\documentclass{article} 
\usepackage{iclr2024_conference,times}
\pdfoutput=1

\usepackage{amsmath,amsfonts,bm}

\usepackage{mathtools}
\usepackage{amssymb}
\usepackage{amsthm}
\usepackage{thmtools, thm-restate}
\usepackage{booktabs}
\usepackage{tikz}
\usetikzlibrary{arrows}
\usepackage{tikz-cd}
\usepackage{adjustbox}

\def\eqref#1{equation~\ref{#1}}

\def\1{\bm{1}}

\newcommand{\giden}[1]{\text{I}_{#1}}
\newcommand{\genlin}[1]{\textnormal{GL}(#1)}
\newcommand{\liehat}{\wedge}
\newcommand{\diag}[1]{\textnormal{diag}(#1)}
\DeclareMathOperator{\id}{id}

\newcommand{\gmatrix}[1]{\textnormal{M}_{#1}(\mathbb{R})}
\newcommand{\groupdexp}[1]{\det{(\frac{1 - e^{-\text{ad}_{-#1}}}{\text{ad}_{-#1}})}}
\newcommand{\groupexpnod}[1]{\frac{1 - e^{-\text{ad}_{-#1}}}{\text{ad}_{-#1}}}
\newcommand{\explie}{\textnormal{expm}} 
\newcommand{\logm}{\textnormal{logm}} 
\newcommand{\rExp}{\textnormal{Exp}}
\newcommand{\Ad}{\textnormal{Ad}}
\newcommand{\ad}{\textnormal{ad}}

\newcommand{\dd}{\textnormal{d}}
\newcommand{\go}[1]{\textnormal{O}(#1)}

\newcommand{\gso}[1]{\textnormal{SO}(#1)}

\newcommand{\grxso}[1]{\sR^{\times}(#1) \times \text{SO}(#1)}
\newcommand{\gse}[1]{\textnormal{SE}(#1, \mathbb{R})}

\newcommand{\gaff}[1]{\textnormal{Aff}(#1)}
\newcommand{\gsl}[1]{\textnormal{SL}(#1, \mathbb{R})}
\newcommand{\ggl}[1]{\textnormal{GL}^{+}(#1, \mathbb{R})}
\newcommand{\gglminus}[1]{\textnormal{GL}^{-}(#1, \mathbb{R})}
\newcommand{\gglfull}[1]{\textnormal{GL}(#1, \mathbb{R})}
\newcommand{\gspd}[1]{\textnormal{Pos}(#1, \mathbb{R})}
\newcommand{\gsspd}[1]{\textnormal{SPos}(#1, \mathbb{R})}
\newcommand{\shgspd}[1]{\textnormal{Pos}(#1)}
\newcommand{\shgsspd}[1]{\textnormal{SPos}(#1)}
\newcommand{\gsym}[1]{\textnormal{Sym}(#1, \mathbb{R})}
\newcommand{\gssym}[1]{\textnormal{Sym}_{0}(#1, \mathbb{R})}
\newcommand{\gutp}[1]{\textnormal{T}(#1, \mathbb{R})}
\newcommand{\gsutp}[1]{\textnormal{ST}(#1, \mathbb{R})}

\newcommand{\lgsl}[1]{\mathfrak{sl}(#1, \mathbb{R})}
\newcommand{\lggl}[1]{\mathfrak{gl}(#1, \mathbb{R})}

\newcommand{\lgso}[1]{\mathfrak{so}(#1)}

\newtheorem{thm}{Theorem}[section]
\newtheorem{defn}[thm]{Definition}

\newtheorem{prop}[thm]{Proposition}
\newtheorem*{assumption}{Assumption}
\newtheorem{remark}[thm]{Remark}

\def\ervv{{\textnormal{v}}}

\def\rmA{{\mathbf{A}}}

\def\rmM{{\mathbf{M}}}

\def\rmR{{\mathbf{R}}}

\DeclareMathAlphabet{\mathsfit}{\encodingdefault}{\sfdefault}{m}{sl}
\SetMathAlphabet{\mathsfit}{bold}{\encodingdefault}{\sfdefault}{bx}{n}

\def\gB{{\mathcal{B}}}

\def\gK{{\mathcal{K}}}
\def\gL{{\mathcal{L}}}

\def\gV{{\mathcal{V}}}

\def\sN{{\mathbb{N}}}

\def\sR{{\mathbb{R}}}

\usepackage{url}
\usepackage{xcolor}
\usepackage{hyperref}
\usepackage{caption}
\usepackage{subcaption}
\usepackage{wrapfig}
\usepackage[toc,page,header]{appendix}
\usepackage{minitoc}

\newcommand{\revision}[1]{{#1}} 

\title{Lie Group Decompositions for Equivariant Neural Networks}

\author{Mircea Mironenco\\
AI4Science Lab, AMLab\\
Informatics Institute\\
University of Amsterdam \\
mircea.mironenco@gmail.com \\
\And
Patrick Forr\'e\\
AI4Science Lab, AMLab\\
Informatics Institute\\
University of Amsterdam \\
p.d.forre@uva.nl \\
}

\iclrfinalcopy 
\begin{document}
\doparttoc 
\faketableofcontents 

\maketitle

\begin{abstract}
Invariance and equivariance to geometrical transformations have proven to be very useful inductive biases when training (convolutional) neural network models, especially in the low-data regime.
Much work has focused on the case where the symmetry group employed is compact or abelian, or both.
Recent work has explored enlarging the class of transformations used to the case of Lie groups, principally through the use of 
their Lie algebra, as well as the group exponential and logarithm maps.
The applicability of such methods is limited by the fact that depending on the group of interest $G$, the exponential map may not be surjective.
Further limitations are encountered when $G$ is neither compact nor abelian.
Using the structure and geometry of Lie groups and their homogeneous spaces, we present a framework by which it is possible to work with such groups primarily focusing on the groups $G = \text{GL}^{+}(n, \mathbb{R})$ and $G = \text{SL}(n, \mathbb{R})$, as well as their representation as affine transformations $\mathbb{R}^{n} \rtimes G$.
Invariant integration as well as a global parametrization is realized by a decomposition into subgroups and submanifolds which can be handled individually.
Under this framework, we show how convolution kernels can be parametrized to build models equivariant with respect to affine transformations\footnote{
Our code is publicly available at \href{https://github.com/mirceamironenco/rgenn}{\texttt{https://github.com/mirceamironenco/rgenn}}.
}.
We evaluate the robustness and out-of-distribution generalisation capability of our model on the benchmark affine-invariant classification task, outperforming previous proposals.
\end{abstract}

\section{Introduction}\label{sec:intro}
Symmetry constraints in the form of invariance or equivariance to geometric transformations have shown to be widely applicable inductive biases in the context of deep learning~\citep{bronstein2021geometric}. Group-theoretic methods for imposing such constraints have led to numerous breakthroughs across a variety of data modalities. CNNs~\citep{lecun1995convolutional} which make use of translation equivariance while operating on image data have been generalized in several directions.
Group-equivariant convolutional neural networks (GCNNs) represent one such generalization~\citet{cohen2016group}.
GCNNs make use of group convolution operators to construct layers that produce representations which transform in a predictable manner whenever the input signal is transformed by an a-priori chosen symmetry group $G$.
These models have been shown to exhibit increased generalization capabilities, while being less sensitive to $G$-perturbations of the input data.
For these reasons, equivariant architectures have been proposed for signals in a variety of domains such as graphs~\citep{han2022geometrically}, sets~\citep{zaheer2017deep} or point clouds data~\citep{thomas2018tensor}.
Constructing equivariant networks entails first choosing a group $G$, a representation for the signal space in which our data lives and a description of the way this space transforms when the group \emph{acts} on it. 
Choosing a particular group $G$ entails making a modelling assumption about the underlying (geometrical) structure of the data that should be preserved. 
Early work has focused on the case where $G$ is finite, with later work largely concentrated on the Euclidean group $\textnormal{E}(n)$, and its subgroups $\gse{n}$ or $\gso{n}$.
Working with continuous groups is much more challenging, and the vast majority of equivariant models focus on the case where the group $G$ has a set of desirable topological and structural properties, namely $G$ is either compact or abelian, or both. 

Recent work~\citep{bekkers2019b, finzi2020generalizing} \revision{explores} the possibility of building equivariant networks for Lie groups - continuous groups with a \revision{smooth} structure.
This research direction is promising since it allows for the modelling of symmetry groups beyond Euclidean geometry.
Affine and projective geometry, respectively affine and homography transformations are ubiquitous within computer vision, robotics and computer graphics~\citep{zacur2014left}. 
Accounting for a larger degree of geometric variation has the promise of making \revision{(vision)} architectures more robust to real-world data shifts.
When working with non-compact and non-abelian Lie groups, for which the group exponential is not surjective, standard harmonic analysis tools cannot be employed directly.
\revision{Our contribution is a framework making it possible to work with such groups.}
\vspace{-2.5mm}
\paragraph{Contributions}
\revision{We present a procedure by which invariant integration with respect to the Haar measure can be done in a principled manner, allowing for an efficient numerical integration scheme to be realized.
We then construct \emph{global} parametrization maps which allow us to map elements back and forth between the Lie algebra and the group, addressing the non-surjectivity of the group exponential.}
We apply our framework to the groups $G = \ggl{n}$ and $G = \gsl{n}$, and more broadly the family of affine matrix Lie groups \revision{$\gaff{G} \coloneqq \sR^{n} \rtimes G$, $G \leq \gglfull{n}$}.
The methodology and tools are generally applicable to any Lie group with finitely many connected components, and we explain how our approach can be seen as a generalization of previous proposals for constructing equviariant layers when working with the regular representation of a topological group.

\section{Related work}\label{sec:related}
Recent proposals for Lie group equivariance~\citep{bekkers2019b, finzi2020generalizing} focus on the infinite-dimensional regular representation of a group and rely on the group exponential map to allow convolution kernels to be defined analytically in the Lie algebra of the group. 
Working with the regular representation entails dealing with an intractable convolution integral over the group, and a (Monte Carlo) numerical integration procedure approximating the integral needs to be employed, which requires sampling group elements with respect to the \emph{Haar} measure of the group.
Unfortunately, the applicability of these methods is limited to Lie groups for which the group exponential map is surjective, which is not the case for the affine group \revision{$\gaff{\gglfull{n}}$}.
These methods also rely on the fact that for compact and abelian groups sampling with respect to the Haar measure of the group is straightforward, which is not the case for the affine groups of interest.
\citet{macdonald2022enabling} propose a framework which can be applied to arbitrary Lie groups, aiming to address such limitations while still relying on the group exponential.
Their approach and its downsides are closely reviewed in Sec.~\ref{subsec:appendix_related}, together with other related equivariant models.

\section{Background}\label{sec:background} 
\paragraph{Continuous group equivariance}
A Lie group $G$ is a group as well as a smooth manifold, such that $\forall g,h \in G$ the group operation  $(g, h) \mapsto gh$ and the inversion map $g \mapsto g^{-1}$ are smooth.
$\gglfull{n}$ \revision{denotes} the Lie group consisting of all invertible $n \times n$ matrices.
A \emph{linear} or \emph{matrix} Lie group refers to a Lie subgroup of $\gglfull{n}$.
$\gglfull{n}$, the translation group $(\sR^{n}, +)$ and the family of 
affine groups \revision{$\gaff{G}$}, 
$G \leq \gglfull{n}$
are our primary interest, with $G$ usually being one of $\ggl{n}$, $\gsl{n} \leq \ggl{n}$ or $\gso{n}$.
Equivariance with respect to the action of a locally compact group $G$ can be realized by constructing layers using the cross-correlation/convolution operators.
We recall that in the continuous 
setting we model our signals as functions $f: X \to \sR^{K}$ defined on some underlying domain $X$.
For example, images and feature maps can be defined as $K$-channel functions $f \in L_{\mu}^{2}(\sR^{2}, \sR^{K})$ which are square-integrable (with respect to the measure $\mu$), and which have bounded support in practice, e.g. $f: [-1, 1]^{2} \subseteq \sR^{2} \to \sR^{K}$.
$\gL_g$ denotes the left-regular representation of $G$, encoding the action
of $G$ on function spaces.
For any continuous $f \in C(X)$:
\begin{align}\label{eq:left_regular_representation}
  [\gL_{g}f](x)\coloneqq f(g^{-1}x),\quad \forall g \in G,~ x \in X
\end{align}
\revision{Every locally compact group $G$ has a left (right) \emph{invariant} Radon measure $\mu_{G}$ called the left (right) \emph{Haar measure} of $G$.}
A canonical example is the translation group $G = (\sR^{n}, +)$ for which $\mu_{G}$ is the Lebesgue measure.
The Haar measure allows for $G$-invariant integration to be realized, and for the 
group convolution to be defined.
\revision{To state the invariance property of $\mu_{G}$, define the 
  functional $\lambda_{\mu_{G}}$:
}
\begin{align}
  \lambda_{\mu_{G}}: L^{1}(G) \to \sR,\quad \lambda_{\mu_{G}}(f) = \int_{G} f(g) \dd{\mu_{G}(g)},~\forall f \in L^{1}(G)
\end{align}
Then, a left Haar measure respects $\lambda_{\mu_{G}}(\gL_{g}f) = \lambda_{\mu_{G}}(f)$ for any $g \in G$ and $f \in L^{1}(G)$.
\revision{Additional details on convolution operators are provided in Sec.~\ref{subsec:appendix_haar}, with Lie groups reviewed in Sec.~\ref{subsec:appendix_liegroups}.}
\vspace{-2.5mm}
\paragraph{Convolution operators}
For a group $G$ acting transitively on locally compact spaces $X$ and $Y$ we then seek to construct 
an operator $\gK: L^2(X) \to L^2(Y)$ satisfying 
the \emph{equivariance constraint} $\gL_g \circ \gK = \gK \circ \gL_g$.
We formalize two scenarios, when $X$ is a homogeneous
space of $G$ (not necessarily a group) and $Y = G$, and the case where $X = Y = G$.
Focusing on the second case, if $\dd{\mu_X} = \dd{\mu_G}$ is the Haar measure on $G$, 
the integral operator $\gK$ can be defined as the standard convolution/cross-correlation.
Let $k: Y \times X \to \sR$ be a kernel that is invariant to the left action of $G$ 
in both arguments, such that $k(gx, gy) = k(x, y)$ for any $(x, y) \in Y \times X$ and $g \in G$.
Let $\mu_X$ be a $G$-invariant Radon measure on $X$, and 
define $\gK \coloneqq C_{k}:L^{p}(X) \to L^{p}(G)$ ($p \in \{1, 2\}$) such that $\forall f \in L^{p}(X)$:
\begin{align}
  C_k: f \mapsto C_{k}f(y) = \int_{X} f(x) k(x, y) \dd{\mu_X(x)},\quad \forall y \in Y
\end{align}
$C_k$ is $G$-equivariant: $\gL_g \circ C_{k} = C_{k} \circ \gL_g,~\forall g \in G$ (\ref{subsec:appendix_equivop}).
Since $X = Y = G$ are homogeneous spaces of $G$ we can easily define a bi-invariant kernel by
projection $k(x, y) = \tilde{k}(g^{-1}_{y}x)$ ($\tilde{k}: G \to \sR$) for any 
$(x, y) \in Y \times X$, where $y = g_y y_0$ for some fixed $y_0$.
The kernel is bi-invariant:
\begin{align}
  k(hx, hy) = \tilde{k}((hg_y)^{-1}hx) = \tilde{k}(g^{-1}_yh^{-1}hx) = \tilde{k}(g^{-1}_yx) = k(x, y),\quad \forall h \in G
\end{align}
For the case $Y = G$ and $g_y = y$ ($y_0 = e$, the identity of $G$) this corresponds to a
cross-correlation.
For a convolution operator, we would analogously define ${k(x, y) = \tilde{k}(g^{-1}_x y)}$ 
where ${x = g_x x_0}$ for ${x_0 \in X}$.
In this case the \emph{essential} component needed 
for equivariance of the operator $C_k$ is the $G$-invariant 
measure $\dd{\mu_X}$, which is the Haar measure when $X = Y = G$.
When $X$ is a homogeneous space of $G$, but not necessarily $G$ itself, we have to work with an operator which takes in a signal in $L^{p}(X)$ and produces a signal $L^{p}(G)$ on the group.
This encompasses the case of the \emph{lifting} layers, which are commonly employed when working with the regular representation of a group~\citep{cohen2016group,kondor2018generalization}.
The kernel $k(\cdot)$ in this case can be derived through an equivariance constraint as in~\cite{bekkers2019b,cohen2019general}.
It can also be shown (\ref{subsec:appendix_equivop}) that an equivariant lifting cross-correlation can be defined as an operator $C^{\uparrow}_{k}$ such that for any $f \in L^{p}(X)$:
\begin{align}\label{eq:lifting_cross_correlation}
  C^{\uparrow}_{k}: f \mapsto C^{\uparrow}_{k}f, \quad C^{\uparrow}_{k}f: g \mapsto \int_{X} f(x) k(g^{-1}x) \delta(g^{-1}) \dd{\mu_X(x)},~\forall g \in G
\end{align}
where $\delta: G \to \sR^{\times}_{>0}$ records the change of variables by the action of $G$ (see \ref{subsec:appendix_equivop}).
Group cross-correlation $C^{\star}_{k} \coloneqq C_{k}$ and convolution 
$C^{*}_{k}$ operators will be defined for any $f \in L^{p}(G)$:
\begin{align}\label{eq:conv_cross_ops}
  C_{k}: f \mapsto C_{k}f, \quad C_{k}f: g \mapsto \int_{G} f(\tilde{g}) k(g^{-1}\tilde{g}) \dd{\mu_G(\tilde{g})},~\forall g \in G\\\label{eq:conv_cross_ops2}
  C^{*}_{k}: f \mapsto C^{*}_{k}f, \quad C^{*}_{k}f: g \mapsto \int_{G} f(\tilde{g}) k(\tilde{g}^{-1}g) \dd{\mu_G(\tilde{g})},~\forall g \in G
\end{align}
\paragraph{Lie algebra parametrization}
The tangent space at the identity of a Lie group $G$ is denoted by $\mathfrak{g}$ 
and called the Lie algebra of $G$.
A Lie algebra is a vector space equipped with 
a bilinear map $[\cdot, \cdot]: \mathfrak{g} \times \mathfrak{g} \to \mathfrak{g}$ called 
the Lie bracket.
To construct an equivariant layer using the Lie algebra of the group, one 
defines the kernels $k(\cdot)$ in (\ref{eq:conv_cross_ops}) or (\ref{eq:conv_cross_ops2}) as functions which take in Lie algebra elements.
This requires a map $\xi: \mathfrak{g} \to G$ which is (at least locally) a
diffeomorphism, with an inverse that can be easily calculated, preferably in closed-form.
This allows us to rewrite the kernel $k: G \to \sR$ as:
\begin{align}
    \label{eq:lie_conv_kernel}
    k(g^{-1}\tilde{g}) = k(\xi(\xi^{-1}(g^{-1}\tilde{g}))) = \tilde{k}_{\theta}(\xi^{-1}(g^{-1}\tilde{g}))
\end{align}
$\tilde{k}_{\theta}(\cdot)$ is effectively 
an approximation of $k(\cdot)$ of the 
form $\tilde{k}_{\theta} \cong k \circ \xi: \mathfrak{g} \to \sR$ 
with learnable parameters $\theta$. 
Using the inverse map $\xi^{-1}(g^{-1}\tilde{g})$, $\tilde{k}_{\theta}$ maps 
the Lie algebra coordinates of the `offset' group 
element $g^{-1}\tilde{g}$ (\revision{for cross-correlations})
to real values corresponding to the evaluation $k(g^{-1}\tilde{g})$.
Our kernels are now maps $\tilde{k}_{\theta} \circ \xi^{-1}: G \to \sR$, 
requiring the implementation of $\xi^{-1}(\cdot)$
and a particular choice for the Lie algebra kernel $\tilde{k}_{\theta}$.
This description encompasses recent proposals for Lie group equivariant layers.
In~\cite{bekkers2019b} the kernels are implemented by modelling 
$\tilde{k}_{\theta}$ via B-splines, 
while~\cite{finzi2020generalizing} choose to parametrize 
$\tilde{k}_{\theta}$ as small MLPs.
Once $\tilde{k}_{\theta}$ and $\xi$ are chosen, we can approximate e.g. the 
cross-correlation using Monte Carlo integration: 
\begin{align}
    \label{eq:monte_carlo_cross_corr}
    \int_{G} f(\tilde{g})\tilde{k}_{\theta}(\xi^{-1}(g^{-1}\tilde{g}))\dd{\mu_{G}}(\tilde{g})
    \approx \frac{\mu_G(G)}{N} \sum_{i=1}^{N} f(\tilde{g}_i)\tilde{k}_{\theta}(\xi^{-1}(g^{-1}\tilde{g}
    _i)),~ \tilde{g}_{i} \sim \mu_G
\end{align}
where $\mu_G(G)$ denotes the volume of the integration space $G$ and $\tilde{
g}_{i} \sim \mu_G$ indicates that $\tilde{g}_{i}$ is sampled (uniformly) with
respect to the Haar measure.
This allows one to obtain equivariance (in expectation) with respect to $G$.
For compact groups, $\mu_G$ can be normalized such that $\mu_G(G) = 1$.
\revision{
  To summarize, we record the components of the framework which 
  are needed for (\ref{eq:monte_carlo_cross_corr}) to realize an equivariant operator.
Namely, we require (1) a \emph{parametrization} map $\xi^{-1}: G \to \mathfrak{g}$, 
as well as (2) the implementation of an efficient \emph{sampling} scheme 
with respect to the Haar measure $\mu_{G}$ such that numerical integration is feasible in practice.}

\section{Lie group decompositions for continuous equivariance}\label{sec:contribution}
\paragraph{Limitations of the group exponential}
\revision{For every Lie group we can define the Lie group exponential map $\explie: \mathfrak{g} \to G$, 
which is a diffeomorphism locally around $0 \in \mathfrak{g}$.} 
Since we are interested in $\gglfull{n}$ and its subgroups, we can make things 
more concrete as follows.
$\gmatrix{nn} \coloneqq \gmatrix{n}$ (the vector space of $n \times n$ real matrices) 
is the Lie algebra of $\gglfull{n}$ (\revision{Sec.~\ref{subsec:appendix_liegroups}}).
The notation $\lggl{n} = \gmatrix{n}$ is used for this identification.
For $G = \gglfull{n}$ with $\mathfrak{g} = \lggl{n}$, the 
group exponential is the matrix exponential $\explie: \lggl{n} \to \gglfull{n}$, 
with the power series expression ${X \mapsto e^{X} = \sum_{k=0}^{\infty} \frac{1}{k!} X^{k}}$.
\revision{The map $\xi$ in (\ref{eq:lie_conv_kernel}) is most commonly implemented as the group exponential $\xi \coloneqq \explie$.
Given a subgroup $G \leq \gglfull{n}$ for which $\explie$ is surjective, every element $g \in G$ can be expressed as ${g = \explie(X) = e^{X}}$ for $X \in \mathfrak{g}$, and fast routines for calculating $\explie(\cdot)$ are available.}
\revision{In this case, the inverse map $\xi^{-1}$ is given by the matrix logarithm, giving us:}
\begin{align}\label{eq:inverse_psi_lieconv}
    \xi^{-1}(g^{-1}\tilde{g}) = \logm(g^{-1}\tilde{g}),\quad \logm: G \to \mathfrak{g}
\end{align}
\revision{In general, we need to consider if both $\xi$ and $\xi^{-1}$ need to be implemented, and whether these maps are available in closed form.}
Assuming there exist $X$ and $Y$ such that $e^{X} = g^{-1}$ and $e^{Y} = \tilde{g}$, (\ref{eq:inverse_psi_lieconv}) can be rewritten as ${\logm(g^{-1}\tilde{g}) = \logm(e^{X}e^{Y})}$.
A key optimization underlying this framework is enabled by employing the BCH formula (\ref{subsec:appendix_liegroups}), which tells us that for \emph{abelian} Lie groups ${\logm(e^{X}e^{Y}) = X + Y}$.
This simplifies calculations considerably and allows one to work primarily at the level of the Lie algebra, bypassing the need to calculate and sample the kernel inputs $g^{-1}\tilde{g}$ at the group level.
Considering the affine Lie groups $\gaff{G},~G \leq \gglfull{n}$, this simplification can be used for example for the abelian groups \revision{$G = \gso{2}$ and $G = \grxso{2}$}, consisting of rotations and scaling.
\citet{bekkers2019b,finzi2020generalizing} primarily work with these groups, and choose $\xi$ and $\xi^{-1}$ to be the matrix exponential and logarithm, respectively.
\revision{If the group is non-abelian but the exponential remains surjective (such as with compact groups like $\gso{3}$), $\explie(\cdot)$ remains a generally valid choice for $\xi$ as long as $\xi^{-1}$ can be accurately calculated in closed-form.} 
For the non-abelian, non-compact groups $\gsl{n}$ or $\ggl{n}$ the non-surjectivity of the exponential map limits the applicability of the matrix logarithm outside of a neighborhood around the identity (Prop.~\ref{prop:matrix_log_existence}).
The class of equivariant networks that can be implemented with this framework is then firstly limited by the parametrization maps $\xi$ and $\xi^{-1}$, motivating the search for an alternative.

Another key limitation is that for (\ref{eq:monte_carlo_cross_corr}) to realize an equivariant estimator when numerically approximating the convolution/cross-correlation integral, 
sampling needs to be realised with respect to the Haar measure of the group $G$.
Techniques for sampling with respect to the Haar measure on the groups $\gso{n}$ or $\grxso{n}$ are known, and generally reduce to working with uniform measures on Euclidean spaces or unit quaternions in the case of $\gso{3}$.
We aim to address these limitations, \revision{allowing the previously described framework to be generalized to arbitrary Lie groups $G \leq \gglfull{n}$. 
We further seek a solution that places minimal limitations on the class of `Lie algebra kernels' ${k_{\theta}: \mathfrak{g} \to \sR}$ that can be used, and one should be able to employ any $k_{\theta}$ that uses the coordinates of tangent vectors in $\mathfrak{g}$ expressed in some basis.}
\revision{In the following we present a set of generally applicable 
tools while considering $\gsl{n}$ and $\ggl{n}$ as working examples, since 
these groups require more consideration and represent our primary application.}
\subsection{Lie group decomposition theory}\label{subsec:liegrdecomp_theory}
We exploit the fact that the groups $\ggl{n}$ and $\gsl{n}$ have 
an underlying product structure that allows them to be decomposed into subgroups 
and submanifolds which are easier to work with individually.
More precisely, ${G \in \{\ggl{n}, \gsl{n}\}}$ can be decomposed 
as a product $P \times H$, where $H \leq G$ is the maximal compact subgroup of $G$ 
and $P \subseteq G$ is a submanifold which is diffeomorphic to $\sR^{k}$, 
for some $k \geq 0$, and we have a diffeomorphism $\varphi: P \times H \to G$.

\revision{
Similar decompositions are available for a larger class of groups $G \leq \gglfull{n}$~\citep[Ch. 6]{abbaspour2007basic}.} 
It can be shown that if the map $\varphi$ is chosen correctly the Haar measure 
$\mu_{G}$ can be expressed as the pushforward measure $\varphi_{\ast}(\mu_{P} \otimes \mu_{H})$,
where $\mu_{P}$ is a $G$-invariant measure on $P$ and $\mu_{H}$ 
is the Haar measure on $H$.
\revision{In some cases the group decomposition presents a corresponding Lie algebra decomposition, which we can leverage to build the parametrization map $\xi^{-1}: G \to \mathfrak{g}$.}
\vspace{-2.5mm}
\paragraph{Factorizing the Haar measure}
Let $G$ be a locally compact group of interest (e.g. $\ggl{n}$), with (left) 
Haar measure $\mu_{G}$.
Assume there exist a set of subspaces or subgroups 
$P \subseteq G$, $K \subseteq G$, such that $G = PK$, and a homeomorphism 
$\varphi: P \times K \to G$.
Further assume that $\mu_{P}$ and $\mu_{K}$ are (left) $G$-invariant 
Radon measures on the corresponding spaces.
We look to express (up to multiplicative coefficients) the Haar 
measure $\mu_G$ as the pushforward of the product 
measure $\mu_P \otimes \mu_K$ under the map $\varphi$. 
This allows for the following change of variables for any $f \in L^{1}(G)$:
\begin{align}
  \int_{G} f(g) \dd{\mu_{G}(g)} = \int_{P \times K} f(\varphi(p, k)) \dd{(\mu_{P} \otimes \mu_{K})}(p, k) = \int_{P} \int_{K} f(\varphi(p, k))\dd{\mu_{K}(k)}\dd{\mu_{P}(p)}
\end{align}
In the context of Monte Carlo simulation this will enable us to produce random
samples distributed according the measure $\mu_G$ by
sampling on the \emph{independent} factor spaces $P$ and $K$ and constructing a sample on $P \times K$
and respectively on $G$ using the map $\varphi$.
The space $P$ will either be another closed subgroup, 
or a measurable subset $P \subseteq G$ that is 
homeomorphic to the quotient space $G/K$.
In particular, if $P$ is not a subgroup, we will focus 
on the case where $P$ is a homogeneous space of $G$ 
with stabilizer $K$ such that $P \cong G/K$.
When the left and right Haar measure of a group coincide, 
the group is called \emph{unimodular}.
The groups $\ggl{n}$, $\gsl{n}$ are unimodular, however this is not 
true for all affine groups \revision{$\gaff{G}$}.
For groups which are volume-preserving, this is not as much of an issue in practice.
However, $\ggl{n}$ is not volume-preserving, and we also desire that our 
framework be general enough to deal with the non-unimodular case as well.
If $G$ is not unimodular and $\mu_{G}$ is its left Haar measure, 
there exists a continuous group homomorphism $\Delta_G: G \to \sR^{\times}_{>0}$, called the \emph{modular function} of $G$, which records the degree to which $\mu_{G}$ fails to be right-invariant.
We now have the tools necessary to record two possible integral decomposition methods.
\begin{thm}
  \label{thm:decompose_haar_abstract}
  (1) Let $G$ be a locally compact group, $H \leq G$ a closed subgroup, with 
    left Haar measures $\mu_{G}$ and $\mu_{H}$ respectively.
    There is a $G$-invariant Radon measure $\mu_{G/H}$ on $G/H$ if and only if
    $\left.\Delta_{G}\right|_{H} = \Delta_{H}$.
    The measure $\mu_{G / H}$ is unique up to a scalar factor 
    and if suitably normalized:
    \begin{align}\label{eq:quotient_integral_formula}
      \int_{G} f(g) \dd{\mu_{G}(g)} =  \int_{G/H} \int_{H} f(gh) \dd{\mu_{H}(h)} \dd{\mu_{G/H}(gH)},~\forall f \in L^{1}(G)
    \end{align}

    (2) Let $P \leq G$, $K \leq G$  closed subgroups such that $G = PK$.
  Assume that $P \cap K$ is compact, and $Z_0$ denotes the stabilizer of
  the transitive left action of $P \times K$ on $G$ given by 
  $(p, k) \cdot g = pgk^{-1}$, for any $(p, k) \in P \times K$ and $g \in G$.
  Let $G$, $P$ and $K$ be $\sigma$-compact 
  (which holds for matrix Lie groups), $\mu_{G}$, $\mu_{P}$ and $\mu_{K}$ left Haar measures 
  on $G$, $P$, and $K$ respectively and $\left.\Delta_{G}\right|_{K} = \Lambda$ 
    is the modular function of $G$ restricted to $K$.
    Then $\mu_{G}$ is given by $\mu_G = \pi_{\ast}(\mu_{P} \otimes \Lambda^{-1}\mu_{K})$, where $\pi: P \times K \to (P \times K) / Z_0$ is the canonical projection.
    In integral form we have:
    \begin{align}
      \int_{G} f(g) \dd{\mu_G(g)} = \int_{P}\int_{K} f(pk) \frac{\Delta_{G}(k)}{\Delta_{K}(k)}\dd{\mu_{K}(k)}\dd{\mu_{P}(p)},\quad \forall f \in L^{1}(G)
    \end{align}
    \begin{proof}
      \citet[Theorem 2.51]{folland2016course} and~\citet[Proposition 7.6.1]{wijsman1990invariant}.
    \end{proof}
\end{thm}
\revision{The existence and range of the convolution operators (for arbitrary Lie groups $G$) are described in Sec.~\ref{subsubsec:conv_range_appendix}, with the non-unimodular case being covered by Prop.~\ref{prop:prop_conv}.
  When going to the Lie group setting, we can already deal 
with semi-direct products of groups of the form $N \rtimes G$.}
The modular function on $N \rtimes G$ is $\Delta_{N \rtimes G}(n, g) = \Delta_N(n) \Delta_G(g)\delta(g)^{-1}$~\citep{kaniuth2013induced}.
The term $\delta: G \to \sR^{\times}_{>0}$ records the effect of the action 
of $G$ on $N$, and it coincides with the term $\delta(\cdot)$ used 
in the lifting layer definition (\ref{eq:lifting_layer_concrete}).
Concretely, take the affine groups $\gaff{G} = \sR^{n} \rtimes G$, 
$G\leq\gglfull{n}$, defined under the semi-direct product structure:
\begin{align}
  \gaff{G} = \sR^{n} \rtimes G = \{(x, A) \mid x \in \sR^n,~A \in G\}
\end{align}
$G$ acts on $\sR^{n}$ by matrix multiplication 
and for $(x, A), (y, B) \in \gaff{G}$, the product and inverse are: 
\begin{align}
  (x, A)(y, B) = (x + Ay, AB), \quad (x, A)^{-1} = (-A^{-1}x, A^{-1})
\end{align}
Elements of $\sR^{n}$ are concretely represented as column vectors.
Viewing $(\sR^{n}, +)$ as the additive group, we 
have $\delta: G \to \sR^{\times}_{>0}$ given by $\delta(A) = \lvert \det(A) \rvert$ for any $A \in G$.
Applying Thm.~\ref{thm:decompose_haar_abstract}, gives:
\begin{align}\label{eq:haar_integral_semidirprod}
  \int_{\gaff{G}} f(g)~\dd{\mu_{\gaff{G}}(g)} = \int_{G} \int_{\sR^{n}}f((x, A))~\frac{\dd{x}\dd{\mu_{G}(A)}}{\lvert \det(A)\rvert},~\forall f \in C_c(\gaff{G})
\end{align}
Expressing the cross-correlation $C_{k}f$ of (\ref{eq:conv_cross_ops2}) in this product space we have for $f \in L^{2}(\gaff{G})$:
\begin{align}\label{eq:cross_correlation_semi_dir_example}
  C_{k}f: (x, A) \mapsto \int_{G} \int_{\sR^{n}} f(\tilde{x}, \tilde{A})k((x, A)^{-1}(\tilde{x}, \tilde{A}))\delta(\tilde{A}^{-1})\dd{\tilde{x}}\dd{\mu_{G}(\tilde{A})}
\end{align}
For the affine groups $\gaff{G} = \sR^{n} \rtimes G$ a 
parametrization map $\xi_{\gaff{G}}: \sR^{n} \oplus \mathfrak{g} \to \gaff{G}$ 
will simply be the identity on the first factor, 
since the Lie algebra of $\gaff{\gglfull{n}}$ decomposes 
as $\sR^{n} \oplus \lggl{n}$ when represented as a Lie subalgebra of $\lggl{n+1}$.
We are then left with the parametrization and invariant 
integration of the $G$-factor of $\gaff{G}$.
We provide a solution for the cases $G = \gsl{n}$ and $G = \ggl{n}$, while remarking 
that a solution for $\ggl{n}$ can be immediately 
extended to $\gglfull{n}$ (Sec.~\ref{subsec:cartanpolar_appendix}).
Our approach is based on a generalized Polar decomposition of matrices, which 
is applicable in the case of reductive ($\ggl{n}$) or semi-simple ($\gsl{n}$) 
Lie groups. An alternative decomposition is discussed in Sec.~\ref{subsec:appendix_qr_fact}.

\paragraph{Manifold splitting via Cartan/Polar decomposition}
Let $\gsym{n}$ be the vector space of ${n \times n}$ real symmetric matrices and
$\gspd{n}$ the subset of $\gsym{n}$ of symmetric positive definite (SPD) matrices.
Denote by $\gsspd{n}$ the subset of $\gspd{n}$ consisting
of SPD matrices with unit determinant, and by $\gssym{n}$ the subspace of 
$\gsym{n}$ of traceless real symmetric matrices.
Any matrix $A \in \gglfull{n}$ can be uniquely decomposed via the left polar decomposition as
$A = PR$ where $P \in \gspd{n}$ and $R\in \go{n}$ (\ref{subsec:cartanpolar_appendix}).
    The factors of this decomposition are uniquely determined and 
    we have a bijection $\gglfull{n} \to \gspd{n} \times \go{n}$ given by:
\begin{align}\label{eq:left_polar_map}
  A \mapsto (\sqrt{AA^{T}}, \sqrt{AA^{T}}^{-1}A),\quad \forall A \in \gglfull{n}
\end{align}
For the reader unfamiliar with Lie group structure theory, the following results 
can simply be understood in terms of matrix factorizations commonly used 
in numerical linear algebra.
The polar decomposition splits the manifold $\ggl{n}$ into the product $\gspd{n} \times \gso{n}$,
and $\gsl{n}$ into $\gsspd{n} \times \gso{n}$.
We use the notation $G \to M \times H$ to cover both cases.
This decomposition can be generalized, as the spaces ${\gspd{n} = \ggl{n} / \gso{n}}$ and 
${\gsspd{n} = \gsl{n} / \gso{n}}$ are actually \emph{symmetric spaces},
and a \emph{Cartan decomposition} is available in this case (\ref{subsec:cartanpolar_appendix}). 
 The Cartan decomposition tells us how to decompose not only at the level of the Lie group, 
 but also at the level of the Lie algebra.
 In fact, using this decomposition we can also obtain a factorization 
 of the measure on these groups.
Let $(G/H, M, \mathfrak{m})$ define our `Lie group data', 
corresponding to $(\ggl{n}/\gso{n}, \gspd{n}, \gsym{n})$ 
or $(\gsl{n}/\gso{n}, \gsspd{n}, \gssym{n})$.
\begin{restatable}{thm}{ManifoldSplittingThm}\label{thm:manifold_splitting_theorem}
Let $(G/H, M, \mathfrak{m})$ be as above, and denote by $\mathfrak{g}$, $\mathfrak{h}$ 
the Lie algebras of $G$ and $H$.
  \begin{enumerate}
    \item The matrix exponential and logarithm are diffeomorphisms between $\mathfrak{m}$ and $M$, 
      respectively. 
      For any $P \in M$ and $\alpha \in \sR$, the power map $P \mapsto P^{\alpha}$ 
      is smooth and can be expressed as:
        \begin{align}\label{eq:spd_power_logexp}
          P^{\alpha} = \explie(\alpha\logm(P)),\quad \forall P \in \gspd{n}
        \end{align}
    \item $G \cong M \times H$ and $G \cong \mathfrak{m} \times H$. We 
      have group-level diffeomorphisms:
      \begin{align}
        \label{eq:cartan_decomposition_diffeomorphism}\chi: M \times H \to G,\quad \chi (P, R) \mapsto PR\\
        \label{eq:cartan_lie_algebra_diffeomorphism}\varPhi: \mathfrak{m} \times H \to G,\quad \varPhi:(X, R) \mapsto \explie(X)R = e^{X}R
      \end{align}
    \item The above maps can be inverted in closed-form:
      \begin{align}\label{eq:cartan_lie_algebra_inverse_diffeomorphism}
        \chi^{-1}&: G \to M \times H,~ \chi^{-1}: A \mapsto (\sqrt{AA^{T}}, \sqrt{AA^{T}}^{-1}A)\\
        \varPhi^{-1}&: G \to \mathfrak{m} \times H,\quad \varPhi^{-1}: A \mapsto (\frac{1}{2}\logm(AA^{T}), \explie(-\frac{1}{2}\logm(AA^{T}))A) 
      \end{align}
    \end{enumerate}
\end{restatable}
A proof and further references for Theorem~\ref{thm:manifold_splitting_theorem} can be found in Sec.~\ref{subsec:manif_split_thm_proof}.
At the level of the Lie algebra, we have the decomposition $\lggl{n} = \lgso{n} \oplus \gsym{n}$.
The Lie algebra of $\gsl{n}$ is ${\lgsl{n} = \{X \in \lggl{n} \mid \textnormal{tr}(X) = 0\}}$.
It decomposes similarly ${\lgsl{n} = \lgso{n} \oplus \gssym{n}}$.
The Cartan decomposition of $\mathfrak{g}$ is therefore expressed as
${\mathfrak{g} = \mathfrak{h} \oplus \mathfrak{m}}$ where ${\mathfrak{h} = \lgso{n}}$ with ${\mathfrak{m} = \gsym{n}}$ if ${G = \ggl{n}}$ and 
${\mathfrak{m} = \gssym{n}}$ if ${G = \gsl{n}}$.
\subsection{A parametrization based on the Cartan Decomposition}
Consider again the notation $(G/H, M, \mathfrak{m})$ 
as in Theorem~\ref{thm:manifold_splitting_theorem} ($G = \ggl{n}$ or $G = \gsl{n}$).
\paragraph{Concrete integral decompositions}
From Theorem~\ref{thm:manifold_splitting_theorem} and the fact that 
symmetric matrices have a unique square root, we actually have equivalent 
decompositions for $A \in G$ as $A = PR$ 
or $A = S^{1/2}R$ for $S, P \in M$, $R \in H$ and $P = S^{1/2}$.
For $\gglfull{n}$, the decomposition $A = S^{1/2}R$, 
has a factorization of the Haar measure of $\gglfull{n}$ as a product 
of invariant measures on $\gspd{n}$ (shortened $\shgspd{n}$) and $\go{n}$.
Let $\mu_{\shgspd{n}}$ denote the $\gglfull{n}$ 
invariant measure on $\shgspd{n}$.
\begin{restatable}{thm}{PolarCoordinateChangeOfVar}\label{thm:polar_coordinates_haar_change}
  Denote $G = \gglfull{n}$, $H = \go{n}$, and 
  let $\mu_{G}$ be the Haar measure on $G$ and $\mu_{H}$ the 
  Haar measure on $H$ normalized by $\textnormal{Vol}(H) = 1$.
  For $A \in G$, under the decomposition $A = S^{1/2}R$, $S \in \shgspd{n}$, $R \in H$, 
  the measure on $G$ splits as $d\mu_{G}(A) = \beta_{n} d\mu_{\shgspd{n}}(S)d\mu_{H}(R)$,
  where $\beta_{n} = \frac{\textnormal{Vol}(\go{n})}{2^{n}}$ is a normalizing constant.
  Restricting to $G = \ggl{n}$ and $H = \gso{n}$ and ignoring constants, 
  we have:
  \begin{align}
    f \mapsto \int_{G} f(A) \dd{\mu_{G}(A)} = \int_{\shgspd{n}}\int_{H}f(S^{1/2}R)\dd{\mu_{H}(R)}\dd{\mu_{\shgspd{n}}(S)},~\forall f \in C_c(G)
  \end{align}
\end{restatable}
The Haar measure of $\gglfull{n}$ is $d\mu_{\gglfull{n}}(A) = \lvert \det(A) \rvert^{-n}dA$, with $dA$ the Lebesgue measure on $\sR^{n^2}$.
We now describe how to sample on the individual factors to obtain $\gglfull{n}$ samples.
\begin{restatable}{thm}{GLnProductSamplingIndep}\label{thm:product_sampling_independence}
If a random matrix ${A \in \gglfull{n}}$ has a left-$\go{n}$ invariant density 
function relative to $\lvert AA^{T} \rvert^{-n/2}dA$, then ${(AA^{T})^{1/2} = S^{1/2}}$ 
and $R = (AA^{T})^{-1/2}A$ are independent random matrices and $R$ 
has a uniform probability distribution on $\go{n}$.
The uniform distribution on $\go{n}$ will 
be the normalized Haar measure $\mu_{\go{n}}$.
Conversely, if $S \in \shgspd{n}$ has a density 
function $f: \shgspd{n} \to \sR_{\geq 0}$ relative to $\mu_{\shgspd{n}}$ and 
$R \in \go{n}$ is uniformly distributed with respect to the 
Haar measure $\mu_{\go{n}}$, then $A = S^{1/2}R$ has a density function $\beta_{n}^{-1}f(AA^{T})\lvert \det(A) \rvert^{-n}$ relative to $dA$.
\end{restatable}
Theorems~\ref{thm:polar_coordinates_haar_change} and~\ref{thm:product_sampling_independence} are known results that appear in the random matrix theory literature, but have not seen recent application in the context of deep learning. 
In (\ref{subsec:appendix_polar_haar}) we provide more details and references.
Using the decomposition $A = S^{1/2}R$ invariant 
integration problems on $G$ can be transferred to the product 
space $M \times H$, and we can express up to normalization the invariant measure 
$\mu_{G}$ as $\varphi_{\ast}(\mu_{M} \otimes \mu_{H})$.
To construct samples $\{\rmA_{1},\ldots,\rmA_{n}\}\sim \mu_{G}$ one produces 
samples $\{\rmR_{1}, \ldots, \rmR_{n}\} \sim \mu_{H}$ where $\mu_{H}$ will 
be the uniform distribution on $H$, and samples 
$\{\rmM_{1}, \ldots, \rmM_{n}\} \sim \mu_{M}$.
Then $\mu_{G}$-distributed random values are obtained 
by $\{\rmA_{1},\ldots,\rmA_{n}\} = \{\varphi(\rmM_{1}, \rmR_{1}), \ldots, \varphi(\rmM_{n}, \rmR_{n})\}$, where again $\varphi: M \times H \to G$ is given by $\varphi: (S, R) \mapsto  S^{1/2}R$.

\paragraph{Mapping to the Lie algebra and back}
Any $A \in G$ can be expressed uniquely as $A = e^{X}R$ 
for $X \in \mathfrak{m}$ and $R \in H$.
Since $H = \gso{n}$ in both cases, the fact that $\explie: \lgso{n} \to \gso{n}$ is 
surjective, allows us to write
it $A = e^{X}e^{Y}$, $Y \in \lgso{n}$.
The factors $X$ and $R = e^{Y}$ are obtained using $\varPhi^{-1}$ 
(\ref{eq:cartan_lie_algebra_inverse_diffeomorphism}).
Then by taking the principal branch of the matrix logarithm  
on $H = \gso{n}$, $Y = \logm(R)$.
A map $\xi^{-1}: G \to \mathfrak{g}$ as described in (\ref{eq:lie_conv_kernel}) 
and (\ref{eq:monte_carlo_cross_corr}) is constructed as $\xi^{-1} = (\id_{\mathfrak{m}} \times~\logm) \circ \varPhi^{-1}$.
More precisely, for any $A = e^{X}e^{Y} \in G$, using $\xi^{-1}$ we obtain 
the tangent vectors 
$(Y, X) \in \lgso{n} \times \mathfrak{m}$ and since 
$\mathfrak{g} = \lgso{n} \oplus \mathfrak{m}$ we have a unique 
$Z = X + Y \in \mathfrak{g}$.
Details are given in (\ref{subsec:lie_algebra_param_appendix}).

Define $\tilde{K}_{\theta} \coloneqq k_{\theta} \circ \xi^{-1}: G \to \sR$ as 
our Lie algebra kernel.
A Monte Carlo approximation of a 
cross-correlation operator $C_{k}: L^{2}(G) \to L^{2}(G)$ as 
in (\ref{eq:monte_carlo_cross_corr}) will be of the form:
\begin{align}\label{eq:cartan_kernel_map_monte_carlo_conv}
  C_{k}f: g \mapsto \frac{1}{N} \sum_{i=1}^{N} f(\tilde{g}_{i})\tilde{K}_{\theta}(\tilde{g}^{-1}_{i}g),~\tilde{g}_{i}\sim \mu_{G},\quad \forall g \in G
\end{align}

For affine groups, every element $(x, A)$ of $\sR^{n} \rtimes G$, can be 
uniquely decomposed as $(x, I)(0, A)$, with $I$ the $n \times n$ identity matrix.
One can use the fact that $\gL_{(x,A)} = \gL_{(x, I)}\gL_{(0, A)}$ to write:
\begin{align}\label{eq:kernel_homomorphism_conv2d}
  k((x, A)^{-1}(\tilde{x}, \tilde{A})) = \gL_{(x, A)}k(\tilde{x}, \tilde{A}) = \gL_{(x, I)}[\gL_{(0, A)}k(\tilde{x}, \tilde{A})] = \gL_{x}[k(A^{-1}\tilde{x}, A^{-1}\tilde{A})]
\end{align}
An efficient implementation of a convolutional layer can be realised in practice 
for $n \in \{2,3\}$ by 
first obtaining the transformed kernel $k(A^{-1}\tilde{x}, A^{-1}\tilde{A})$ and 
then applying the translation $\gL_{x}$ using an efficient convolution routine, 
as done for example in~\citet{cohen2016group, bekkers2019b}.

\revision{In practice, the exact discretization of the translation factor $\sR^{n}$ will depend on the support of the input data.}
For example, if our input signals are defined compactly on a grid (e.g. 2D images), 
we can approximate a continuous convolution~\citep{finzi2020generalizing} by sampling 
the translation factor in a uniform grid of coordinates $\tilde{x} \sim [-1, 1]^{n} \subset \sR^{n}$ as the parametrization $\xi_{\gaff{G}}: \sR^{n} \times \lggl{n} \to G$ is 
the identity map for the first factor.
We can then approximate a lifting cross-correlation layer by:
\begin{align}\label{eq:lifting_layer_concrete}
  C^{\uparrow}_{k}f: (x, A) &\mapsto \int_{\sR^{n}} f(\tilde{x})\gL_{x}k(A^{-1}\tilde{x})\delta(A^{-1})\dd{\tilde{x}}\\
    &\approx \frac{1}{N} \sum_{i=1}^{N}f(\tilde{x}_{i})\gL_{x}[k_{\theta}(A^{-1}\tilde{x}_{i})\delta(A^{-1})],~x_i \sim [-1, 1]^{n}
\end{align}
For the non-lifting layers, starting from 
(\ref{eq:cross_correlation_semi_dir_example}), denoting $\dd{\tilde{A}} = \dd{\mu_{G}(\tilde{A})}$ and applying (\ref{eq:kernel_homomorphism_conv2d}) we have:
  \begin{align}
    [C_{k}f](x, A) = \int_{\sR^{n}} \int_{G} f(\tilde{x}, \tilde{A}) \gL_{x}k(A^{-1}\tilde{x},A^{-1}\tilde{A})\delta(\tilde{A}^{-1}) \dd{\tilde{x}}\dd{\tilde{A}}
  \end{align}
  Using Theorem~\ref{thm:polar_coordinates_haar_change}, denote the invariant measures $\mu_{M}$ and $\mu_{H}$ by $\dd{S}$ and $\dd{R}$, we obtain: 
  \begin{align}
   [C_{k}f](x, A) = \beta_{n} \int_{\sR^{n}} \int_{H} \int_{M} f(\tilde{x},S^{1/2}R)\gL_{x}k(A^{-1}\tilde{x},A^{-1}S^{1/2}R)\delta(S^{-1/2}) \dd{S}\dd{R}\dd{\tilde{x}}
  \end{align}
  The kernel in (\ref{eq:cartan_kernel_map_monte_carlo_conv}) is 
  now of the form $K_{\theta}: \sR^{n} \rtimes G \to \sR$, giving us:
  \begin{align}\label{eq:semidir_final_lie_kernel_form}
    [C_{k}f](x, A) \approx \frac{V}{N}\sum_{i=1}^{N}f(\tilde{x_{i}}, S_{i}^{1/2}R_{i})\gL_{x}[K_{\theta}(A^{-1}\tilde{x_{i}},\xi^{-1}(A^{-1}S_{i}^{1/2}R_{i}))\delta(S_{i}^{-1/2})]
  \end{align}
  where $\tilde{x_i} \sim [-1, 1]^{n}$, $S_i \times R_i \sim (\mu_{M} \otimes \mu_{H})$, 
  and $R_i$ sampled uniformly with respect $\mu_{H}$.
  $V$ records both the volume of the integration space from the MC approximation as well as the constant $\beta_{n}$.

\section{Experiments}\label{sec:experiments}
For all experiments we use a ResNet-style architecture, replacing convolutional layers with cross-correlations that are equivariant (in expectation) with respect to the groups ${\sR^{2} \rtimes \ggl{2}}$ and ${\sR^{2} \rtimes \gsl{2}}$.
Details regarding the network architecture and training 
are given in Appendix~\ref{sec:experiment_details}. 
\vspace{-2.5mm}
\paragraph{Affine-transformation invariance}
We evaluate our model on a benchmark affine-invariant image classification task 
employing the affNIST dataset\footnote{\url{http://www.cs.toronto.edu/~tijmen/affNIST}}.
The main works we compare with 
are the affine-equivariant model of~\cite{macdonald2022enabling} and the 
Capsule Networks~\citet{de2020introducing, ribeiro2020capsule} which are state 
of the art for this task.
The experimental setup involves training on the standard set of $50000$ 
non-transformed MNIST images (padded to $40\times 40$), and evaluating on 
the affNIST test set, which consists of $320000$ affine-transformed 
MNIST images.
The model never sees the transformed affNIST images during training, and we 
do not use any data augmentation techniques.
In this case, robustness with respect to the larger groups of 
the affine family of transformations is needed.
For a fair comparison we roughly equalize the number of parameters 
with the referenced models.
\begin{table}[h]
    \small
    \centering
    \caption{affNIST classification accuracy, after training on MNIST.}
    \label{table:affnist_equivariance}
    \scalebox{0.9}{
    \begin{tabular}{ccccc}
        \toprule
        Model & affNIST Acc. & MNIST Acc. & Parameters & MC. Samples  \\
        \midrule
        $\sR^{2} \rtimes \gsl{2}$ & $98.5(\pm 0.1)$ & $99.55 (\pm 0.1)$ & $370$K & $10$\\
        \midrule
        VB CapsNet~\citep{ribeiro2020capsule} & $98.1$ & $99.7$ & $175$K & --- \\
        \midrule
        RU CapsNet~\citep{de2020introducing} & $97.69$ & $99.72$ & $>580$K & --- \\
        \midrule
        $\sR^{2} \rtimes \ggl{2}$ & $97.4(\pm 0.2)$ & $99.5 (\pm 0.1)$ & $395$K & $10$\\
        \midrule
        affConv~\citep{macdonald2022enabling} & $95.08$ & $98.7$ & $374$K & $100$\\
        \midrule
        affine CapsNet~\citep{gu2020improving} & $93.21$ & $99.23$ & --- & ---\\
        \midrule
        Equivariant CapsNet~\citep{lenssen2018group} & $89.1$ & $98.42$ & $235$K & ---\\
        \midrule
    \end{tabular}
  }
\end{table}

Table \ref{table:affnist_equivariance} reports the average test performance of 
our model at the final epoch, over five training runs with different initialisations.
We observe that our equivariant models are robust and generalize well, 
with the $\sR^{2} \rtimes \gsl{2}$ model outperforming all 
previous equivariant models and Capsule Networks.
Note that, compared to~\citet{macdonald2022enabling}, our sampling 
scheme requires $10$ times less samples to realize an accurate Monte Carlo 
approximation of the convolution.
The $\sR^{2} \rtimes \ggl{2}$ model performs slightly worse than 
the volume-preserving affine group $\sR^{2} \rtimes \gsl{2}$.
This can be explained by considering that the affNIST dataset contains 
only a small degree of scaling.
\paragraph{Homography transformations}
We further evaluate and report 
in Table~\ref{table:homnist_equivariance_mnist_generalization} the performance 
of the same model evaluate on the homNIST dataset of~\citet{macdonald2022enabling}.
The setup is identical to the affNIST case, with the images now being 
transformed by random homographies.
We observe a similar degree of robustness in this case, again outperforming previous methods applied to this task.
\begin{table}[ht]
    \small
    \centering
    \caption{homNIST classification.}
    \label{table:homnist_equivariance_mnist_generalization}
    \scalebox{1.0}{
    \begin{tabular}{ccc}
        \toprule
        Model & homNIST Acc. & MC. Samples\\
        \midrule
        $\sR^{2} \rtimes \gsl{2}$ & $98.3(\pm 0.1)$ & $10$\\
        \midrule
        $\sR^{2} \rtimes \ggl{2}$ & $97.71(\pm 0.1)$ & $10$\\
        \midrule
        affConv~\citep{macdonald2022enabling} & $95.71$ & $100$\\
        \midrule
    \end{tabular}
  }
\end{table}

\revision{
  As our models are only equivariant in expectation, we analyze numerically in Sec.~\ref{sec:equivariance_error_sec} the degree to which the \emph{equivariance error} is dependent on the number of Monte Carlo samples used to approximate the convolution/cross-correlation integral.
}

\section{Conclusion}\label{sec:conclusion} 
We have built a framework for constructing equivariant networks 
when working with matrix Lie groups that are not necessarily compact or abelian.
Using the structure theory of semisimple/reductive Lie groups we have shown one 
possible avenue for constructing invariant/equivariant (convolutional) 
layers primarily relying 
on tools which allow us to decompose larger groups into smaller ones.
In our preliminary experiments, the robustness and 
out-of-distribution capabilities of the equivariant 
models were shown to outperform previous proposals on tasks where the symmetry 
group of relevance is one of $\ggl{n}$ or $\gsl{n}$.

Our contribution is largely theoretical, providing a framework by which equivariance/invariance to complex symmetry groups can be obtained.
Further experiments will look to validate the applicability of our method to other data modalities, such as point clouds or molecules, as in~\citet{finzi2020generalizing}.

While we have primarily focused on convolution operators, we remark that the 
tools explored here are immediately applicable to closely-related 
machine learning models which employ Lie groups and their regular representation 
for invariance/equivariance. 
For example, the `LieTransformer' architecture 
proposed in~\cite{hutchinson2021lietransformer} opts to replace convolutional layers 
with self-attention layers, while still using the Lie algebra of the group 
as a mechanism for incorporating positional information.
They face the same challenge in that their parametrization is dependent 
on the mapping elements back and forth between a chosen Lie group and its Lie algebra, 
and they require a mechanism for sampling on the desired group.
The methods presented here are directly applicable in this case. 

Future work will explore expanding the class of Lie groups 
employed by such models using the tools presented here.
Another potential avenue to explore is the applicability of the presented 
tools to the problem of `partial' and `learned' invariance/equivariance~\citep{benton2020learning}.
The sampling mechanism of the product decomposition allows one to specify a probability distribution for the non-orthogonal factor, which could be learned from data.

\section*{Acknowledgments}
The presentation of this paper at the conference was financially supported
by the Amsterdam ELLIS Unit and Qualcomm.

\bibliography{references.bib}
\bibliographystyle{iclr2024_conference}

\newpage
\appendix
\part{Appendix} 
{
  \hypersetup{linkcolor=black}
  \parttoc
}
\section{Additional background}\revision{
A symmetry group refers to a set of transformations which preserve some underlying structure present in the data.
Formally, a group $G$ is a set together with an associative binary operation $G \times G \to G$ which tells us how group elements can be composed to form another.
Every element $g \in G$ has an inverse $g^{-1} \in G$, and the group has an identity element $e \in G$ such that $gg^{-1} = e$.
}

\subsection{Topological groups and the Haar measure}\label{subsec:appendix_haar}
A Hausdorff\footnote{A topological space where any two distinct points 
have disjoint neighbourhoods is called Hausdorff.} topological 
space $M$ is \textbf{locally compact} if each point $m \in M$ 
has a compact neighborhood.
A \textbf{topological group} $G$ is a group as well as a Hausdorff 
topological space such that the group operation $(g, h) \mapsto gh$ and inversion
map $g \mapsto g^{-1}$ are continuous.
For the following and more details on Radon measures 
see~\citet{folland1999real, lang2012real}.
\paragraph{Radon measures}
  Let $(X, \gB_{X}, \mu)$ be a measure space with Hausdorff topological space $X$,
  $\gB_{X}$ the $\sigma$-algebra of Borel sets and $\mu: \gB_{X} \to [0, \infty]$ any measure on $\gB_{X}$ 
  (referred to as a Borel measure).
  The measure $\mu$ is called \textbf{locally finite} if every point $x \in X$ has
  an open neighborhood $U \ni x$ for which $\mu(U) < \infty$.
  A Borel set $E \subseteq X$ ($E \in B_X$) is called \textbf{inner regular} if:
  \begin{align}
    \mu(E) = \sup\{\mu(K) \mid K \subseteq E,\text{ $K$ compact}\}
  \end{align}
  Respectively, a Borel set $E \subseteq X$ is called \textbf{outer regular} if:
  \begin{align}
    \mu(E) = \inf\{\mu(V) \mid V \supseteq E,\text{ V open}\}
  \end{align}
  The measure $\mu$ is called \textbf{inner} (\textbf{outer}) regular if all Borel
  sets are inner (outer) regular.
  It is called \textbf{regular} if it is both inner and outer regular.
  $\mu$ is a \textbf{Radon measure} if it is locally finite, inner regular on
    open sets and outer regular.
  $\mu$ is \textbf{$\sigma$-finite} if there exists a countable family of
    Borel sets $\{B_n\}_{n \in \sN}$, where $\mu(B_n) < \infty, ~\forall n \in N$ and
    $\bigcup_{n \in N} B_n = X$.
If $\mu$ is a Borel measure on a Hausdorff topological space $X$, local finiteness
will imply that $\mu$ is finite on compact subsets of $X$.
\begin{thm}\footnote{See Proposition 2.9 and Theorems 2.10 \& 2.20 of \citet{folland2016course}.}\label{thm:haar_measure_exists}
    Every locally compact group $G$ has a left (right) nonzero Radon
    measure $\mu_G$ such that $\mu_G(gB) = \mu_G(B)$
    (respectively $\mu_G(Bg) = \mu_G(B)$) for any Borel subset
    $B \subseteq G$ and any $g \in G$.
    The measure $\mu_G$ is called the left (right) Haar measure of $G$ and
    if $\nu_G$ is another Haar measure on $G$, then $\mu_G = c \cdot \nu_G$ for some 
    $c \in \sR_{>0}$.
\end{thm}
When integrating with respect to the left Haar measure $\mu_G$ we have 
for any $f \in C_c(G)$:
\begin{align}\label{eq:haar_invariance_int}
    \int_G f(yx) \dd{\mu_G} = \int_G f(x) \dd{\mu_G},~\forall y \in G
\end{align}
All of the groups and topological spaces in the main text will be $\sigma$-compact.
A locally compact space $X$ is $\sigma$-compact (or countable at infinity) if it 
is a countable union of compact subsets.
We will use the notation $\mu_G$ to refer to the left Haar measure and when needed
the notation $\mu_L(\cdot)$ and $\mu_R(\cdot)$ will be used to differentiate left and
right Haar measures.
\begin{remark}
  If $X$ is a homogeneous space of $G$ (but not $G$ itself), 
  then a $G$-invariant Radon measure $\dd{\mu_{X}}$ on $X$ (if it exists) respects 
  the same invariance property presented in Thm.~\ref{thm:haar_measure_exists} and (\ref{eq:haar_invariance_int}), and we simply refer to $\dd{\mu_{X}}$ as a $G$-invariant measure.
  For a review of such measures, see~\cite[Chapter 2.6]{folland2016course}.
\end{remark}
\paragraph{Function spaces}
Suppose $(X, \gB_{X}, \mu)$ is a measure space where $X$ is locally compact 
Hausdorff space and $\mu$ a Radon measure on $X$.
$L_{\mu}^p(X, \sR) \coloneqq L^p(X)$ for $1 <= p < \infty$ denotes 
the space of equivalence classes of functions $\{f: X \to \sR \mid f \text{ Borel measurable and } \int_X \lvert f \rvert^p \dd{\mu} < \infty \}$
that agree $\mu$-almost everywhere.
Equipped with the norm $\|f\|_{p} = (\int_{X} \lvert f \rvert^{p} \dd{\mu})^{1/p}$, $L^{p}(X)$ is a Banach space.
$C(X) \coloneqq C(X, \sR)$ denotes the space of continuous
real-valued functions on $X$.
The support of $f \in C(X)$ is defined as $\text{supp}(f) = \overline{\{x
\in X \mid f(x) \neq 0\}}$.
We state that a function $f$ has compact support whenever $\text{supp}(f)$ is
compact.
$C_c(X)$ denotes the subspace of $C(X)$ of continuous functions with
compact support.

\subsection{Equivariant convolutional operators}\label{subsec:appendix_equivop}
\revision{In this section we show that the convolution/cross-correlation operators (\ref{eq:conv_cross_ops}), (\ref{eq:conv_cross_ops2}) as well as the lifting cross-correlation (\ref{eq:lifting_cross_correlation}) are equivariant.
We then clarify the existence and range of these operators.}
For a non-lifting convolution/cross-correlation operator, recall that $X = Y = G$ and
$k: Y \times X \to \sR$ is a bi-invariant kernel:
\begin{align}
  k(gx, gy) = k(x, y),\quad \forall (x, y) \in Y \times X,~\forall g \in G
\end{align}
$\mu_X$ is a $G$-invariant Radon measure on $X$. 
The operator $C_{k}$ was defined for any $f \in L^{1}_{\mu_{X}}(X)$:
\begin{align}
  C_k: f \mapsto C_{k}f(y) = \int_{X} f(x) k(x, y) \dd{\mu_X(x)},\quad \forall y \in Y
\end{align}
$C_{k}: f \mapsto C_{k}f$ maps an element of $L^{1}(X)$ to $L^{1}(G)$, respecting $\gL_g \circ C_{k} = C_{k} \circ \gL_g$ for any $g \in G$:
\begin{align} 
  \gL_g(C_{k}f)(y) = C_{k}f(g^{-1}y) &= \int_{X} f(x) k(x, g^{-1}y) \dd{\mu_X(x)} \\
    (\text{Haar invariance})~&= \int_{X} f(g^{-1}x) k(g^{-1}x, g^{-1}y) \dd{\mu_X(x)} \\
  (\text{Kernel bi-invariance})&= \int_{X} f(g^{-1}x) k(x, y) \dd{\mu_X(x)} = C_k(\gL_g[f])(y)
\end{align}

A lifting cross-correlation operator $C^{\uparrow}_{k}$ was defined for $f \in L^{1}(X)$ by:
\begin{align}
  C^{\uparrow}_{k}: f \mapsto C^{\uparrow}_{k}f, \quad C^{\uparrow}_{k}f: g \mapsto \int_{X} f(x) k(g^{-1}x) \delta(g^{-1}) \dd{\mu_X(x)},~\forall g \in G
\end{align}
where $X$ is a homogeneous $G$-space with Radon measure $\dd{\mu_X}$ and 
$\delta: G \to \sR^{\times}_{>0}$ records the change of variables produced by the action of $G$:
\begin{align}
  \int_{X} f(x) \dd{\mu_{X}(x)} = \int_{X} f(gx) \delta(g) \dd{\mu_{X}(x)}
\end{align}
\revision{In the case where $X = \sR^{n}$ and $Y = \gaff{G}$, for $\gaff{G}$ the semi-direct 
product ${\gaff{G} = \sR^n \rtimes G}$, $G \leq \gglfull{n}$,
we have $\gaff{G}$ acting transitively on $X$, and we can 
make the identification ${X \cong \gaff{G} / G}$.
Each element $g \in \gaff{G}$ can be represented as 
$g = (x, h)$ with $x \in \sR^{n}, h \in G$.}
In this case, we have $\delta(g) = \lvert \det(g) \rvert = \lvert \det(h) \rvert$.
Note that the form of the kernel can equivalently be derived 
through an equivariance constraint relation as in~\citet[Theorem 1]{bekkers2019b}.
The lifting layer is then more concretely of 
the form $C^{\uparrow}_{k}f = f \star k$ where:
\begin{align}
  (f \star k)(g) = \int_{\sR^n} f(x) k(g^{-1} x) \frac{\dd{x}}{\lvert \det(h) \rvert}
\end{align}
The lifting layer is equivariant 
($\gL_{\tilde{g}}[f \star k] = [\gL_{\tilde{g}}f] \star k$), since for any 
$\tilde{g} = (\tilde{x}, \tilde{h}) \in \gaff{G}$ we have:
\begin{align}
  \gL_{\tilde{g}}[f \star k](g) = (f \star k)(\tilde{g}^{-1}g) &= \int_{\sR^{n}} f(x) k(g^{-1}\tilde{g}x)\frac{\dd{x}}{\lvert \det(\tilde{h}^{-1}h) \rvert} \\
  (x \mapsto\tilde{g}^{-1}x)&=\int_{\sR^{n}} f(\tilde{g}^{-1}x)k(g^{-1}x)\lvert \det(\tilde{h}^{-1}) \rvert \frac{\dd{x}}{\lvert \det(\tilde{h}^{-1}h) \rvert} \\
&= \int_{\sR^{n}} \gL_{\tilde{g}}f(x) k(g^{-1}x) \frac{\dd{x}}{\lvert \det(h) \rvert}= ([\gL_{\tilde{g}}f] \star k)(g)
\end{align}
A similar derivation appears in~\citet[Theorem 4.2]{macdonald2022enabling}.
\revision{The lifting layer equivariance here was derived for the case where $X$ is the 
homogeneous space $\sR^{n} = \gaff{G} / G$, however a similar process is available
for more general homogeneous spaces. One only has to identify 
the appropriate \emph{relatively} invariant measure $\delta(\cdot)^{-1}\dd{\mu_{X}}$ which appears in (\ref{eq:lifting_cross_correlation}). For more details on relatively invariant measures see~\citet{hewitt2012abstract}.}

\revision{
\subsubsection{Existence and range of convolution operators}\label{subsubsec:conv_range_appendix}
}
If $G$ is a locally compact Hausdorff group, $C_c(G)$ is dense 
in $L^{p}(G)$ for $1 \leq p < \infty$\footnote{\cite[Proposition 7.9]{folland1999real}.}.
We can therefore approximate functions in $L^{p}(G)$ using functions in $C_c(G)$.
While some results hold in a more general setting, we assume
that all topological groups are (locally) compact Hausdorff and second countable 
(and therefore $\sigma$-compact), 
as the Lie groups of interest satisfy these properties.
\begin{prop}\footnote{\cite[Propositions 2.40 \& 2.41]{folland2016course}.}\label{prop:prop_conv}
We record the following results concerning the existence and range of 
the convolution operators for a locally compact group $G$.
  \begin{enumerate}
  \item If $f \in L^{1}(G)$ and for any $k \in L^{p}(G)$ ($1 <= p <= \infty$) then 
    $\int_G f(\tilde{g})k(\tilde{g}^{-1}g)\dd{\mu_G(\tilde{g})}$ converges 
      absolutely, $f * k \in L^{p}(G)$ and $\|f * k\|_p \leq\|f\|_1\|k\|_p$.
    \item If $G$ is not unimodular, $f \in L^{p}(G)$ and 
      $k \in L^{1}(G) \cap C_c(G)$ then $f * k \in L^{p}(G)$.
    \item If $f, k \in L^{2}(G)$ then $f * k \in C_0(G)$,
      see also~\cite[Theorem 20.16]{hewitt2012abstract}.
  \end{enumerate}
\end{prop}
For the cross-correlation operator, if we define the involution 
$f^{*}(g) = f(g^{-1})$, we can write $f \star k$ as $f * k^{*}$, and reuse
the results of Proposition~\ref{prop:prop_conv}.
In the case of a nonunimodular semi-direct product group $G = N \rtimes H$, where
$H$ \emph{is} unimodular, the Haar measure is of 
the form $\mu_{G} = \Delta_{G}\cdot\mu_{N} \otimes \mu_{H}$\footnote{\revision{Corresponds to Thm.~\ref{thm:decompose_haar_abstract}(2) in integral form. A more precise reference is~\citet[Corollary 7.6.3]{wijsman1990invariant}.}}.
Since $\Delta_G: G \to (0, \infty)$ is unbounded, we therefore always make the 
assumption that the support of $fk$ is a 
compact set where $\Delta_G$ is bounded.
\begin{assumption}\label{assum:compat_support}
We assume $f, k \in L^{1}(G) \cap L^{2}(G)$ and additionally
$k \in C_c(G)$.
When working with image data, one can also take $f \in C_c(G)$ directly.
Going forward we therefore establish for $p \in \{1, 2\}$, 
$C^{\uparrow}_{k}: L^{p}(X) \to L^{p}(G)$ defined 
by (\ref{eq:lifting_cross_correlation}) and $C^{*}_{k}$ or $C_{k}$ are 
similarly operators $L^{p}(G) \to L^{p}(G)$ given by (\ref{eq:conv_cross_ops}) 
and (\ref{eq:conv_cross_ops2}).
\end{assumption}

The restriction to a compact \emph{subset} of $G$ can also be motivated if we
wish to employ a Monte Carlo approximation of the integral using a uniform 
distribution, since the Haar measure is finite on compact subsets.
However, the restriction will not tied to the injectivity/surjectivity radius of the
group exponential map, as we will use a different parametrization 
as shown in the main text.

\revision{\subsection{Related work}\label{subsec:appendix_related}}
The theory behind constructing equivariant convolutional layers when our input space is a homogeneous space of some locally compact topological group is covered in~\citet{cohen2016group,kondor2018generalization,cohen2019general, bekkers2019b}.
A differential geometric formulation not necessarily limited to homogeneous spaces
is given in~\citet{weiler2021coordinate, weiler2023EquivariantAndCoordinateIndependentCNNs} and a review focused on 
the application of induced representations in the context of
neural networks can be found in~\citet{kondor2018generalization, cohen2019general}.

Employing Monte Carlo integration to approximate convolution integrals has also been proposed e.g. in~\citet{hermosilla2018monte,finzi2020generalizing,romero2022ckconv,knigge2022exploiting}.
In~\cite{hutchinson2021lietransformer} the Lie algebra parametrization is employed with convolution operators being replaced with self-attention layers.
Steerable CNNs make use of group representation theory to parametrise convolution kernels by solving a kernel steerability constraint~\citet{cohen2017steerable, weiler2019general, lang2021wigner, cesa2021ENsteerable}.

In the context of finite-dimensional group representation~\citet{finzi2021practical} present a general solution for constructing equivariant MLPs. They present a framework for solving the equivariance constraint by making use of the generators of the Lie algebra of a group. The resulting linear system is solved for finite dimensional representations by using the singular value decomposition. More general solutions are developed in~\citet{bogatskiy2020lorentz,batatia2023general} for a wider class of Lie groups and representations.

\paragraph{\cite{macdonald2022enabling}}
The limitations of previous Lie algebra methods (as reviewed in Sec.~\ref{sec:contribution}) are also discussed in~\cite{macdonald2022enabling}, which proposes 
a possible solution while still working with the group exponential.
To overcome its lack of surjectivity and be able to sample with respect 
to the Haar measure of $\gaff{\gglfull{n}} = \sR^{n} \rtimes \gglfull{n}$,
the domain of integration itself is restricted and the convolution integral 
is reformulated to ensure that the group elements $\tilde{g}^{-1}g$ (using the convolution operator notation)
are within the injectivity radius of the exponential map.
This is done by first changing the convolution operator after the lifting layer
to work with the right Haar measure\footnote{See~\cite[Theorem 4.3]{macdonald2022enabling} or~\cite[(2.32) \& (2.36)]{folland2016course}. 
Here, we used the convolution operator to be 
consistent with the derivation of~\cite{macdonald2022enabling}.
}:
\begin{align}\label{eq:right_haar_cross}
  (f * k)(g) = \int_{G} f(\tilde{g})k(\tilde{g}^{-1}g)\dd{\mu_{L}(\tilde{g})} = \int_{G} f(g\tilde{g}^{-1}) k(\tilde{g}) \dd{\mu_{R}(\tilde{g})}
\end{align}
For non-lifting layers instead of the kernel $k(\cdot)$, 
the feature map $f(\cdot)$ is now evaluated at $f(g\tilde{g}^{-1})$. 
\begin{prop}\footnote{\cite[Theorem 1.14]{helgason1984groups},
  ~\cite[Theorem 4.4]{macdonald2022enabling}.}\label{prop:haar_canonical_coordinates}
  Let $G$ be a Lie group with Lie algebra $\mathfrak{g}$ and suppose 
  $U \subseteq \mathfrak{g}$ is a neighborhood of $0 \in \mathfrak{g}$ and $\explie(U)$
  a neighborhood in $e \in G$ such that the group exponential $\explie: \mathfrak{g} \to G$
  is a diffeomorphism of $U$ onto $\explie(U)$.
  For $f \in C_c(G)$ with support in $\explie(U)$ we have:
\begin{align}
    \label{eq:group_exp_cov}
    \int_{G} f(g) \dd \mu_G(g) = \int_{\mathfrak{g}} f(\explie(X))\groupdexp{X} \dd{X}
\end{align}
where $\groupdexp{X}$ is the Jacobian determinant of the differential of $\explie$
and $\dd{X}$ is the Lebesgue measure on $\mathfrak{g}$.
\end{prop}
The change of variables of Proposition~\ref{prop:haar_canonical_coordinates} and 
the expression (\ref{eq:right_haar_cross}) allow~\cite{macdonald2022enabling} 
to define the non-lifting convolutional layers to be of the form:
\begin{align}\label{eq:macdonald_convolution}
  (f * k)(g) = \int_{G} f(g\tilde{g}^{-1}) k(\tilde{g}) \dd{\mu_{R}(\tilde{g})} = \int_{\mathfrak{g}} f(ge^{-X})\tilde{k}_{\theta}(X)\groupdexp{X}\dd{X}
\end{align}
with $\tilde{k}_{\theta}$ again a learnable map that takes in Lie algebra elements
approximating $k \circ \explie: \mathfrak{g} \to \sR$.

The difference here lies in the fact that one does not need to use the inverse 
map $\xi^{-1}$ to map back to $\mathfrak{g}$.
By treating $\groupdexp{X}$ as (proportional to) a density function, sampling 
is realised directly on the Lie algebra using standard MCMC methods.
One starts the sampling process from the outer-most convolution integral and 
rejects samples that lie outside of the support of the exponential map.
While the approach can be applied to any matrix Lie group, 
in practice because every possible $\tilde{g}^{-1}g$ element must be 
precalculated and kept in memory before the forward pass,
the scalability of the method is greatly limited due to its memory requirements.
Further numerical errors are introduced due to the fact that $\groupdexp{X}$ 
is approximated with limited precision using the power series expression $\groupexpnod{X} = \sum_{k=0}^{\infty}\frac{(-1)^{k}}{(k + 1)!}(\ad_X)^{k}$.
Additionally, while a discretization of the integral is always necessary,
this approach is limited in that the domain of integration must still be restricted
to the injectivity/surjectivity radius of the group exponential for the change of variables 
to apply.
In this paper, we also employ a change of variables in the context of invariant integration with respect to the Haar measure.
However, rather than working with the tangent space of the group as in Prop.~\ref{prop:haar_canonical_coordinates}, 
we make use of group-level decompositions into independent factor spaces, and show that the Haar measure decomposes as a product of invariant measures, allowing us to construct a sampling scheme on the lower-dimensional subcomponents, as explained in Sections~\ref{subsec:liegrdecomp_theory} and~\ref{subsec:appendix_polar_haar}.

\section{Lie group decompositions}\subsection{Lie groups}\label{subsec:appendix_liegroups}
A \textbf{Lie group} $G$ is a group as well as a smooth manifold, such that both 
the group operation and the inversion map are smooth.
Lie groups are therefore second countable Hausdorff topological spaces.
An abelian Lie group $G$ is a Lie group and an abelian group, i.e. a group for which 
the order of the group operation does not matter $gh = hg,~\forall g,h \in G$.
$\gmatrix{nn} \coloneqq \gmatrix{n}$ denotes the vector 
space of $n \times n$ matrices. 
It is canonically isomorphic to $\sR^{n^2}$, which is locally compact.

Closed or open subsets of $\gmatrix{n}$ will be locally compact 
with respect to the induced topology\footnote{\cite[Proposition 4.66]{lee2010introduction}.}. 
One such open subset is $\gglfull{n}$, the Lie group of invertible matrices:
\begin{align}
    \gglfull{n} = \{X \in \gmatrix{n} \mid \det(X) \neq 0\}
\end{align}
The notation $H \leq G$ ($H < G$) is used 
  to indicate that $H$ is a (proper) subgroup of $G$, 
rather than just a (proper) subset $H \subseteq G$ ($H \subset G$).
We are only interested in closed Lie subgroups of $\gglfull{n}$.

A (closed) \emph{Lie subgroup}\footnote{
An abstract subgroup $H$ of $G$ is a submanifold iff $H$ is closed 
with the induced topology (\cite[Theorem 19.18]{gallier2020differential}).
} $H$ of a Lie group $G$ will refer to 
a closed subgroup and a submanifold of $G$ (with the induced topology).
A \emph{linear} or \textbf{matrix Lie group} is defined to be a Lie subgroup
of $\gglfull{n}$, and will therefore be locally compact and second countable.
$\gglfull{n}$, the translation group $(\sR^{n}, +)$ and the family of 
affine groups $\sR^{n} \rtimes H$, 
$H \leq \gglfull{n}$ are our primary interest, with $H$ being one of the groups:
\begin{itemize}
  \item $\ggl{n} = \{X \in \gglfull{n} \mid \det(X) > 0\}$, the identity component 
    of $\gglfull{n}$;
    It it also referred to as the \emph{positive} general linear group;
  \item $\gsl{n} = \{X \in \gglfull{n} \mid \det(X) = 1\}$, the special linear group;
  \item $\go{n} = \{X \in \gglfull{n} \mid X^{T}X = \giden{n}\}$, the orthogonal group;
  \item $\gso{n} = \{X \in \go{n} \mid \det(X) = 1\}$, the special orthogonal group.
\end{itemize}

\begin{prop}\footnote{\cite[Proposition 20.8]{lee2013smooth}.}\label{prop:lie_exp_prop}
  The group exponential map $\explie: \mathfrak{g} \to G$ is smooth with 
  $d(\explie)_{0} = \id$, making $\explie$ a diffeomorphism $\explie|_{U}: U \to V$ 
  of some neighborhood $U$ of $0 \in \mathfrak{g}$ onto a neighborhood $V$ of $e \in G$.
\end{prop}

The notation $\explie: \mathfrak{g} \to G$ for the Lie group exponential is used 
to differentiate it from the Riemannian exponential, which is mentioned later on.
Unless the group $G$ can be equipped with a bi-invariant Riemannian metric, 
the exponentials do not coincide (see Proposition~\ref{prop:bi_metric_exp_properties}).

$\gmatrix{n}$ equipped with the matrix commutator 
$[X, Y] = XY - YX$ for $X, Y \in \gmatrix{n}$ is a Lie algebra, and more precisely 
it is (canonically isomorphic to) the Lie algebra of 
$\gglfull{n}$\footnote{\cite[Example 8.36 \& Proposition 8.41]{lee2013smooth}.}.
We use the notation $\lggl{n} = \gmatrix{n}$ when working with 
this identification.
For $G = \gglfull{n}$ with $\mathfrak{g} = \lggl{n}$, the 
group exponential is given by the matrix exponential:
\begin{align}
  \explie: \lggl{n} \to \gglfull{n},\quad X \mapsto e^{X} = \sum_{k=0}^{\infty} \frac{1}{k!} X^{k}
\end{align}
From Prop.~\ref{prop:lie_exp_prop} we can define the 
inverse of the group exponential $(\explie|_{U})^{-1}: V \to U$ which is a diffeomorphism 
of $V$ onto $U$.
For matrix Lie groups this map is the matrix logarithm which we denote by 
$\logm(\cdot)$.
Its power series expression is:
\begin{align}
  \logm(A) = \sum_{i=1}^{\infty} \frac{(-1)^{k+1}}{k}(A - I)^{k}, \quad A \in \gglfull{n}
\end{align}

The existence of the inverse of the matrix exponential is characterized as follows. 
\begin{prop}\footnote{\cite[Theorem 2.2.1]{faraut2008analysis}.}\label{prop:matrix_log_existence}
  Let $B(I, 1) = \{X \in \gmatrix{n} \mid \|X - I\| < 1\}$
  where $\|\cdot\|$ is a norm on $\gmatrix{n}$ (e.g. Frobenius norm) and $I$ 
  the identity matrix.
  Note that $B(I, 1) \subseteq \gglfull{n}$. Then for every $g \in B(I, 1)$ we 
  have $\explie(\logm(g)) = g$ and for every $X \in B(0, \log(2))$, 
  we have $\logm(\explie(X)) = X$.
\end{prop}

Recall from Section~\ref{sec:contribution} that if the parametrization map $\xi: \mathfrak{g} \to G$ is chosen to be the group exponential $\xi = \explie$, then $\xi^{-1}$ is given by the matrix logarithm:
\begin{align}
    \xi^{-1}(g^{-1}\tilde{g}) = \logm(g^{-1}\tilde{g}),\quad \logm: G \to \mathfrak{g}
\end{align}
For the following, see~\citet[Chapter 5]{hall2015lie}.
Assuming there exist $X$ and $Y$ such that $e^{X} = g^{-1}$ and
$e^{Y} = \tilde{g}$, (\ref{eq:inverse_psi_lieconv})
can be rewritten as $\logm(e^{X}e^{Y})$.
The Baker-Campbell-Hausdorff (BCH) formula states 
that there exists a sufficiently small open subset
$0 \in U \subset \mathfrak{g}$ so that
$e^{X}e^{Y} \in e^{U}$ and one has:
\begin{align}\label{eq:BCH_approx}
  \logm(e^{X}e^{Y})=X+Y+\frac{1}{2}[X, Y]+\frac{1}{12}[X,[X, Y]]-\frac{1}{12}[Y,[X, Y]]+\ldots
\end{align}
For abelian Lie groups this reduces to:
\begin{align}
    \label{eq:BCH_kernel}
    \logm(e^{X}e^{Y}) = X + Y
\end{align}

For $V$ a finite-dimensional vector space over $\sR$, we denote by $\genlin{V}$ the
group of invertible linear maps of $V$ and by 
$\lggl{V} = \textnormal{End}(V)$ the space of linear maps $V \to V$.
$\genlin{V}$ admits a Lie group structure as it is isomorphic to $\gglfull{n}$ once
a basis is chosen.
The space $\lggl{V}$ can be made into a Lie algebra under 
the commutator bracket and it is isomorphic to $\gmatrix{n} = \lggl{n}$.

\begin{defn}
  Let $V$ be a vector space.
  A (finite-dimensional real) representation of a Lie group $G$ is a 
  Lie group homomorphism $\rho: G \to \genlin{V}$.
  For $\mathfrak{g}$ a (real) Lie algebra, a representation of $\mathfrak{g}$ is a 
  Lie algebra homomorphism $\phi: \mathfrak{g} \to \lggl{V}$.
\end{defn}
The conjugation (inner automorphism) map $C_g: G \to G$ 
is defined such that $C_g = L_g \circ R_{g^{-1}}$:
\begin{align}
    C_g: G \to G,\quad C_g: h \mapsto ghg^{-1}
\end{align}
The \textbf{adjoint representation} $G$, is given 
by the homomorphism\footnote{\cite[Proposition 20.24]{lee2013smooth}.}
\begin{align}
  \Ad: G \to \genlin{\mathfrak{g}},~g \mapsto \Ad_g
\end{align}
where $\Ad_{g}: \mathfrak{g} \to \mathfrak{g},~\Ad_{g} = d(C_g)_e = (dL_g)_{g^{-1}} \circ (dR_{g^{-1}})_e$.
The differential of $\Ad$ is used to define the \textbf{adjoint representation}
of $\mathfrak{g}$, denoted by $\ad$:
\begin{align}
\ad: \mathfrak{g} \to \mathfrak{gl}(\mathfrak{g}),~\ad_X = d(\Ad)_e(x)
\end{align}

\subsection{Primer on Riemannian Geometry}
For more details on smooth and Riemannian manifolds see~\citet{lee2013smooth}.
\paragraph{Riemannian manifolds}
A Riemannian metric (tensor) $g$ on a smooth manifold $M$ is a covariant 2-tensor field smoothly assigning to each $p \in M$, an inner product $g_p: T_p M \times T_p M \to \sR$ at its tangent space $T_p M$:
\begin{align}
p \mapsto g_p(\cdot, \cdot) = \langle \cdot, \cdot \rangle_p
\end{align}
A smooth manifold $M$ with a Riemannian metric $g$ is a \emph{Riemannian manifold} $(M, g)$.
\begin{defn}\label{def:riemannian_isometry}
Let $(M, g)$ and $(N, h)$ be Riemannian manifolds.
A map $\phi: M \to N$ is an \textbf{isometry} if $\phi$ is a
diffeomorphism and $g = \phi^{\ast}h$.
Equivalently, $\phi$ is bijective, smooth and $\forall p \in M$,
$d\phi_p: T_p M \to T_{\phi(p)} N$ is a linear isometry:
\begin{align}
    g_p(u, v) = h_{\phi(p)}(d\phi_p(u), d\phi_p(v)),~\forall u, v \in T_p M
\end{align}
\end{defn}
An \emph{affine connection} is a bilinear map $\nabla$ that 
maps a pair of vector fields $X, Y$ to another vector field $\nabla_{X}Y$, which is the 
covariant derivative of $Y$ with respect to $X$.
The affine connection allows us to define the notion of a parallel vector field.
If $M$ is a smooth manifold and $\nabla$ a connection on $M$, 
then for $\gamma: I \to M$ a smooth curve, any vector field $X$ along $\gamma$ 
is called parallel if $\nabla_{\dot{\gamma}(t)}X = 0$ for any $t \in I$, where 
$\dot{\gamma}\left(t_0\right) \coloneqq d \gamma_{t_0}\left(\left.\frac{d}{d t}\right|_{t_0}\right)$.
A smooth curve $\gamma: I \to M$ is a \emph{geodesic} (with respect to $\nabla$) 
iff $\dot{\gamma}(t)$ is parallel along $\gamma$, 
that is $\nabla_{\dot{\gamma}(t)}\dot{\gamma}(t) = 0,~\forall t \in I$.
For every point $p \in M$ and every tangent vector $v \in T_p M$ there
exists some interval ${I = (-\eta, \eta)}$, ${\eta > 0}$ around $0$ and a
unique geodesic $\gamma: I \to M$ satisfying:
\begin{align}\label{eq:geodesic_initial_conditions}
    \gamma(0) = p,~\text{and}~\gamma^{\prime}(0) = v
\end{align}
There exists a unique geodesic $\gamma$ satisfying
these conditions and for which the domain $I$ cannot be extended.
In this case, $\gamma$ is the unique \textbf{maximal geodesic}
satisfying the initial conditions (\ref{eq:geodesic_initial_conditions}).
We denote it by $\gamma_{p, v}$, and say that $\gamma_{p, v}$ is 
a geodesic through $p$ with initial velocity $v$.
\begin{defn}[Exponential map of connection]\label{def:exponential_of_connection}
    Let $M$ be a manifold and $\nabla$ a connection on $M$.
    Define for a point $p \in M$ the set $D(p) = \{v \in T_p M \mid \gamma_{p, v}(1)~\text{defined}\}$.

    The exponential map $\exp_p: D(p) \to M$ is given by:
    \begin{align}
        \exp_p: v \mapsto \gamma_{p, v}(1)
    \end{align}
\end{defn}
\begin{prop}\footnote{\cite[Propositions 16.4 \& 16.5]{gallier2020differential}.}\label{prop:riem_exp_props}
The differential of the exponential map
$d(\exp_p)$ at $0$ is the identity on
$T_p M$.
For every $p \in M$, the exponential $\exp_p$ is a diffeomorphism
from an open subset\\ $U \subseteq T_p M$ centered at $0$ such that
$\exp_p(U) \subseteq M$ is open.
\end{prop}
It is therefore possible to build a local chart ($\exp_p(U)$, $\exp^{-1}_p$) 
around every point $p \in M$ using the inverse of the exponential map $\exp^{-1}_p$.
The Levi-Civita connection is the unique affine connection
that is metric-compatible and torsion free.
\begin{defn}\footnote{\cite[Definitions 16.8-16.10]{gallier2020differential}.}\label{def:riemannian_exp_log}
  The exponential map of the Levi-Civita connection will be called
  the \textbf{Riemannian exponential}.
  For any $p \in M$, a \textbf{normal neighborhood} of $p$ is an open neighborhood
  $U_p = \exp_p(B(0, \epsilon))$ where $\exp_p$ is a diffeomorphism from the open ball
  $B(0, \epsilon) \subseteq T_{p}M$ onto $U_p$.
  The \textbf{injectivity radius} $\textnormal{inj}_{M}(p)$ at $p$ is the least
  upper bound value $\epsilon > 0$ such that $\exp_p$ is a diffeomorphism 
  on $B(0, \epsilon)$.
  The chart ($\exp_p(B(0, \textnormal{inj}_M(p)))$, $\exp^{-1}_p$) is called a 
  \textbf{normal chart} and the inverse of the exponential $\exp^{-1}_p$ is 
  the \textbf{Riemannian logarithm}.
\end{defn}

\subsubsection{Lie groups as Riemannian manifolds}\label{subsec:lie_group_riemannian}

A Riemannian metric on a Lie group $G$ is called \emph{left-invariant} iff the 
left-translation map is an isometry:
  \begin{align}
      \langle u, v \rangle_{x} = \langle (dL_a)_x u, (dL_a)_x v \rangle_{ax},~\forall a,x \in G,~\forall u, v \in T_x G
  \end{align}
\textbf{Right-invariant} metrics are defined analogously.
A \emph{bi-invariant metric} on $G$ is a Riemannian metric
that is both left and right invariant.
To specify an invariant metric we can use the following.
\begin{prop}\footnote{\cite[Propositions 21.1 \& 21.2]{gallier2020differential}}
  \label{prop:left_inv_metric_inner_prods}
  For $G$ a Lie group with Lie algebra $\mathfrak{g}$ there is a one-to-one 
  correspondence between inner products on 
  $\mathfrak{g}$ and left (right) invariant metrics on $G$.
\end{prop}
Left (right) invariant metrics can therefore be determined uniquely by specifying
an inner product on $\mathfrak{g}$. 
This can be seen since for any $x \in G$ and $u, v \in T_x G$:
\begin{align}
    \label{eq:group_left_inv_metric}
    \langle u, v \rangle_x = \langle \dd (L_{x^{-1}})_{x} u, \dd (L_{x^{-1}})_{x} v \rangle_{e}
\end{align}
The analogue result for bi-invariant metrics is the following.
\begin{prop}\footnote{\cite[Proposition 21.3]{gallier2020differential}.}
  There is a one-to-one correspondence between $\Ad$-invariant inner products 
  on $\mathfrak{g}$ and bi-invariant metrics on $G$.
  An $\Ad$-invariant inner product on $\mathfrak{g}$ is defined 
  such that for any $g \in G$, $\Ad_g$ is a linear isometry:
  \begin{align}
      \langle u, v \rangle = \langle \Ad_{g}(u), \Ad_{g}(v) \rangle,~\forall g \in G,~\forall u, v \in \mathfrak{g}
  \end{align}
\end{prop}
Lie groups with bi-invariant metrics are convenient to work with 
as the group and Riemannian exponential maps coincide at the identity.
\begin{prop}\footnote{\cite[Propositions 21.6 \& 21.20]{gallier2020differential}.}\label{prop:bi_metric_exp_properties}
  If a Lie group $G$ is compact, then it has a bi-invariant metric.
  If $G$ has a bi-invariant metric, then the Riemannian exponential at the identity 
  $\exp_e: T_e G \to G$ coincides with the 
  Lie group exponential $\explie: \mathfrak{g} \to G$. 
\end{prop}
The Lie groups of interest $\gsl{n}$, $\gglfull{n}$ or $\gse{n}$ 
do not admit bi-invariant Riemannian metrics\footnote{A connected Lie group $G$ has a bi-invariant metric iff it is isomorphic to the Cartesian product of a compact group and an additive vector group $(\sR^{n}, +)$ for $n \geq 0$ (\cite[Theorem 21.9]{gallier2020differential}).}.
When only a left (right) invariant metric is available, it is still 
possible in some cases to obtain closed-form expressions for geodesics such as 
the Riemannian exponential.
Suppose the group $\gglfull{n}$ is endowed with the canonical 
left-invariant metric determined by the inner product $\langle X, Y \rangle \coloneqq \textnormal{tr}(X^{T}Y)$ for $X, Y \in \lggl{n}$. Then for any $A \in \gglfull{n}$:
\begin{align}\label{eq:left_invariant_inner_prod_gln}
  g_{A}(X_1, X_2) = g_{I}(A^{-1}X_1, A^{-1}X_2) = \langle A^{-1}X_1, A^{-1}X_2 \rangle,\quad \forall X_1, X_2 \in T_{A}\gglfull{n}
\end{align}
A closed-form expression for the Riemannian exponential map at the identity 
is given by\footnote{See~\cite{andruchow2014left, martin2014minimal}.
Note that in this case, the metric will be left invariant and right-$\go{n}$-invariant.
}: 
\begin{align}\label{eq:riem_exp_identity_gln}
  \exp_{I}(X) = e^{X^{T}}e^{X - X^{T}}, \quad \forall X \in \lggl{n}
\end{align}
On the right-hand side we have used the Lie group exponential, 
given by the matrix exponential.
The same expression holds 
for $\gsl{n}$\footnote{See~\cite[Section 8.9]{zacur2014left} which 
references~\cite{wang1969discrete,helgason1979differential}.} and can be used to 
define the exponential at any point.

\begin{remark}\label{rem:riemannian_logarithm_choice}
If we equip $\ggl{n}$ or $\gsl{n}$ with their canonical left-invariant metric, 
the Riemannian exponential is available 
in closed form given by (\ref{eq:riem_exp_identity_gln}).
As opposed to the Lie group exponential, the Riemannian exponential \emph{is}
surjective, and one could see it as a possible choice for 
the parametrization map $\xi: \mathfrak{g} \to G$.

However, as explained in Section~\ref{sec:contribution}, if one cannot work 
only at the level of the Lie algebra, and group elements 
have to be mapped from the group $G$ to $\mathfrak{g}$, the map $\xi^{-1}$ 
would also need to be available.
We are not aware of a closed-form expression for the 
Riemannian logarithm on the 
groups $\ggl{n}$, $\gsl{n}$, corresponding to the canonical left-invariant metric.

One could employ a \emph{shooting} or \emph{relaxation} method to compute the 
Riemannian logarithm\footnote{See~\citet[Section 9]{zacur2014left} 
for a general discussion on obtaining the Riemannian logarithm in the context 
of matrix Lie groups.}, 
as done for example in~\cite[Chapter 6.2]{rentmeesters2013algorithms}.
However, since $\xi^{-1}$ is used at every (lifting) cross-correlation layer, 
this would add a large computational cost during the forward pass.
This issue motivated the search for an alternative solution, 
as the one proposed in the main text.
\end{remark}

\subsection{Group actions \& Homogeneous spaces}\label{subsec:appendix_homspace}
Let $G$ be a group and $M$ a set.
The left action of $G$ on $M$ is a map
$\lambda: G \times M \to M$ satisfying for any $m \in M$ and group elements $h, g \in G$:
    \begin{align}
      \lambda(e, m) = m\text{ and } \lambda(h, \lambda(g, m)) = \lambda(hg, m)
    \end{align}
Right actions are defined analogously.
Where it is clear we are referring to a left action we use the
notation $g \cdot m$ or $gm$ for $\lambda(g, m)$.
Using the action we can define the map
$\lambda_g: M \to M$ given by $\lambda_g: x \mapsto g \cdot x$.
Since Lie groups are locally compact topological groups
some results will be useful if stated more generally.
If $G$ is a topological group and $M$ a topological
space, then $\lambda$ is continuous and $\lambda_g$ will be a homeomorphism.
Whereas when $G$ is a Lie group and $M$ a smooth manifold the
action is smooth and $\lambda_g$ will be a diffeomorphism.
The action of $G$ on $M$ is \textbf{transitive} if:
\begin{align}
    \forall x, y \in M:~\exists g \in G: g \cdot x = y
\end{align}
If a group $G$ acts transitively on a set $M$, then $M$ is
called a homogeneous space of $G$.
For any point $x \in M$, the set of group elements that fix $x$ form
  a subgroup of $G$ called the \textbf{isotropy group} or stabilizer of $x$, 
  and denoted by $G_{x}$.
  The \textbf{orbit} of a point $x \in M$ is denoted by $O_{x}$:
\begin{align}
  \label{eq:stabilizer}G_x = \{g \in G \mid g \cdot x = x\}\\
  \label{eq:orbit}O_x = G \cdot x = \{g \cdot x \mid g \in G\} \subseteq M
\end{align}
Let $G$ also act on a set $N$. 
A map $f: M \to N$ is \textbf{equivariant} if it commutes with the action of $G$:
\begin{align}\label{eq:equivarint_basic}
  f(g \cdot m) = g \cdot f(m),\quad \forall m \in M,~\forall g \in G 
\end{align}
\begin{prop}\footnote{See for example~\cite[Proposition 5.2]{gallier2020differential}.}\label{prop:passing_quotient}
Let $\lambda: G \times M \to M$ be a transitive left action of a group $G$ 
on a set $M$, and denote by $H = G_x$ the stabilizer of $x \in M$.
The map by $\pi: G \to G/H$ denotes the canonical projection $\pi: g \mapsto gH$ 
on the left cosets for any $g \in G$.
For any such $x \in M$ the projection (or orbit) map $\pi_x: G \to M$ is a 
surjective map defined by:
\begin{align}
  \pi_x: G \to M,~\pi_x: g \mapsto \lambda(g,x) = g \cdot x
\end{align}
Since $\pi_x$ is surjective and we have
$\pi_x(gH) = gH \cdot x = g \cdot Hx = g \cdot x = \pi_x(g)$ for any $g \in G$,
it induces a bijection $\phi_x: G/H \to M$ by passing to the quotient:
\begin{align}\label{eq:set_quotient_bijection}
  \pi_x = \phi_x \circ \pi,\quad \phi_x: \pi(g) \mapsto g \cdot x
\end{align}
\end{prop}
\begin{thm}\footnote{\cite[Proposition 5.7 \& Theorems 5.13-5.14]{gallier2020differential}.}\label{thm:homogeneous_space_topo}
  Let $G$ be a locally compact Hausdorff group that is 
  also $\sigma$-compact.
  Suppose $G$ acts transitively 
  and continuously on a locally compact Hausdorff space $M$.
  For any $x \in M$, the stabilizer $G_x$ is a closed subgroup of $G$ and
  the quotient space $G / G_x$ is Hausdorff.
  Denoting $G_{x} = H$, the projection $\pi: G \to G/H$ is a continuous open map, and the orbit map $\pi_x: G \to M$
  is also continuous.
  Furthermore, the map $\phi_x: G / H \to M$ is a homeomorphism, 
  and it is $G$-equivariant, where the action of $G$ on $G/H$ 
  is defined as in (\ref{eq:g_action_gh})\footnote{\cite[Theorem 0.4.5]{abbaspour2007basic}.}.
\end{thm}
If $M$ and $N$ are smooth manifolds, $\pi: M \to N$ a smooth map, 
and ${d\pi_p: T_{p}M \to T_{\pi(p)}N}$ its differential at $p \in M$.
  $\pi$ is a \textbf{smooth submersion} if $d\pi_{p}$ is surjective for every $p \in M$.
  The subset $\pi^{-1}(x)$ for any $x \in N$ is referred to as the \textbf{fiber (of $\pi$) over $x$}, and it is a properly embedded submanifold\footnote{\cite[Corollaries 5.13 \& 5.14]{lee2013smooth}.}. 
The `analogue' of 
Theorem~\ref{thm:homogeneous_space_topo} in the Lie group setting 
are the following results.
\begin{thm}\footnote{\cite[Theorem 21.17]{lee2013smooth}.}\label{thm:homogeneous_space_maniffold}
    Suppose $H$ is a closed Lie subgroup of a Lie group $G$.
    There exists a unique smooth structure on the set of 
        left cosets $G / H$ so that the canonical 
        projection $\pi: G \to G / H$ is 
        a smooth submersion.
      Furthermore, the left action of $G$ on the cosets:
        \begin{align}\label{eq:g_action_gh}
        \tau: G \times G / H \to G / H,~(g_1, g_2H) \mapsto g_1g_2H
    \end{align}
    is transitive and smooth,~i.e. $G /H$ is a homogeneous $G$-space.
\end{thm}
$G / H$ is also referred to as a \emph{coset manifold}.
Note that $\pi: G \to G/H$ is $G$-equivariant, and we have 
a diffeomorphism $\tau_h: G/H \to G/H,~\tau_{h}: gH \mapsto hgH$ such that: 
\begin{align}\label{eq:projection_coset_action_relation}
    \tau_g \circ \pi = \pi \circ L_g,\quad \tau_{gh} = \tau_g \circ \tau_h,\quad \forall g,h \in G
\end{align}
\begin{thm}\footnote{See \cite[Theorem 21.18]{lee2013smooth} and~\cite[Proposition 11.13]{o1983semi}.}\label{thm:equivariant_diffeomorphism}
    Let $\lambda: G \times M \to M$ be a smooth transitive left action of a Lie
    group $G$ on a smooth manifold $M$.
    For any $x \in M$, let $G_x = H$ denote its stabilizer.
      The stabilizer $H$ is a closed Lie subgroup of $G$.
        $\phi_x: G / H \to M$ defined as in (\ref{eq:set_quotient_bijection}):
       \begin{align}
            \phi_x: G/H \to M,~gH \mapsto \lambda(g, x) = g \cdot x
       \end{align}
    is a diffeomorphism.
    Furthermore, $\phi_x$ is equivariant with respect
to the action of $G$ on $G / H$ and the action of
    $G$ on $M$.
  The projection map $\pi_x: G \to M,~g \mapsto \lambda(g, x) = g \cdot x$ (which 
    can be expressed as $\pi_{x} = \phi_x \circ \pi$) is a smooth submersion.
\end{thm}

In the main text, we employ the decomposition of 
a group $G$ as $G/H \times H$ for $H \leq G$ a closed subgroup.
The following sections describe the specific class of homogeneous 
spaces $G/H$ for which these decompositions are realised.
Our starting point is the following general result.

\begin{prop}\footnote{
The same result is obtained in the case of topological groups when 
$\sigma: G \to M$ is a continuous 
cross section of $\pi$, in which case $\varphi: M \times H \to G$ is 
then a homeomorphism.}\label{prop:global_cross_section}
  Let $G$ be a Lie group, $H$ a closed Lie subgroup, and denote $M = G / H$.
  If the projection $\pi: G \to M$ has a smooth cross section 
  $\sigma: M \to G$ ($\pi \circ \sigma = \id_{M}$) then:
  \begin{align}\label{eq:bimeasurable_forward}
    \varphi: M \times H \to G,~\varphi(m, h) = \sigma(m)h
  \end{align}
  defines a diffeomorphism from the product space $M \times H$ onto $G$.
  \begin{proof}
    The proof is given in~\cite[Lemma 11.16]{o1983semi}. 
    Here we give a more verbose
    description of the construction since the inverse of this map is mentioned
    in the following sections.
    Let $H = G_x$ be the stabilizer of a point $x \in M$, where $M = G/H$.
    To show that $\varphi: M \times H \to G$ is a diffeomorphism 
    we define the inverse map $\psi: G \to M \times H$ such that:
    \begin{align}\label{eq:bimeasurable_inverse}
      \psi: g \mapsto (\pi(g), (\sigma(\pi(g)))^{-1}g),\quad \forall g \in G
    \end{align}
    $\psi$ is smooth as it is a composition of smooth maps.
    To show that $\psi$ is well-defined one shows $\sigma(\pi(g))^{-1}g \in H = G_x$.
    Recall from Proposition~\ref{prop:passing_quotient} the projection 
    $\pi_x(g) = g \cdot x$ for any $g \in G$ and that
    we've assumed $\pi \circ \sigma = \id_{M}$:
    \begin{align}
      \sigma(\pi(g)) \cdot x = \pi_x(\sigma(\pi(g))) = \pi_x(g) = g \cdot x
    \end{align}
    Then 
    $(\sigma(\pi(g)))^{-1}g \cdot x = (\sigma(\pi(g)))^{-1}\sigma(\pi(g))\cdot x = x$,
    such that $\sigma(\pi(g))^{-1}g \in H = G_x$. 
    We therefore have $\psi(g) \in M \times H$, and it is clear that 
    $\varphi \circ \psi = \id_{G}$ and $\psi \circ \varphi = \id_{M \times H}$, 
    the result follows.
  \end{proof}
\end{prop}
Lemma 11.27 of \cite{o1983semi} gives a method for constructing such a map 
for a class of homogeneous spaces $M = G/H$ called \emph{naturally reductive}.
\revision{Before reviewing these spaces, we need to define \emph{Riemannian submersions}.
These class of submersions will allow us to describe the 
geometry of $G/H$ using the geometry of $G$.
For a comprehensive description one can consult~\citet[Chapter 7]{o1983semi}
or~\cite[Section 18.3]{gallier2020differential}, which serve as our main references.

Suppose $M$ and $N$ are smooth manifolds, and $\pi: M \to N$ a submersion.
For any $x \in \pi(M)$, the fiber above $x$ given by $\pi^{-1}(x)$
is a submanifold of $M$.
For any $p \in \pi^{-1}(x)$ then ${T_p \pi^{-1}(x) = \ker d\pi_p}$.
Any complement of $\ker d\pi_p = T_p \pi^{-1}(x)$ in $T_p M$ will be isomorphic to
$T_{\pi(p)}N$.
In the Riemannian case for $(M, g)$ and $(N, h)$ smooth manifolds endowed with metrics,
the fibers $\pi^{-1}(x)$ will be Riemannian submanifolds of $M$, and we
can define an orthogonal decomposition with respect to the metric.
The orthogonal subspaces are referred to as horizontal and vertical subspaces.
More precisely, for each $x \in \pi(M) \subseteq N$ and $p \in \pi^{-1}(x)$, the tangent space $T_p M$ can be
decomposed into orthogonal subspaces ${T_p M = \ker d\pi_p \oplus (\ker d\pi_p)^{\bot} = V_p \oplus H_p}$.
Tangent vectors $v \in T_p M$, can be written uniquely using horizontal and vertical components:
\begin{align}
    \label{eq:tangent_vector_decomp}
    v = v_H + v_V, \quad v_H \in H_p, v_V \in V_p
\end{align}
If $v \in H_p$ ($V_p$), then $v$ is called a \textbf{horizontal (vertical)
  tangent vector}.
The differential $d\pi_p$ of a submersion being surjective for any $p \in M$ 
allows us to construct a vector space isomorphism  $d\pi_p\vert_{H_p}: H_p \to T_{\pi(p)}N$ 
between horizontal spaces $H_p$ and $T_{\pi(p)}N$.
$\pi$ is a \textbf{Riemannian submersion} if for all $p \in M$, the differential
$\dd\pi_p$ restricted to $H_p$ is a linear isometry onto $T_{\pi(p)}N$:
\begin{align}
  \label{eq:isometry_riem_subm}
  g_p(u, v) = h_{\pi(p)}(d\pi_p(u), d\pi_p(v)), \quad \forall u,
  v \in H_p
\end{align}

The main utility of Riemannian submersions in our case comes from the next theorem 
which describes how to express geodesics in $N$ as projections of horizontal 
geodesics in $M$.
\begin{thm}\footnote{See~\cite[Proposition 18.8]{gallier2020differential} 
or~\cite[Lemma 7.45 \& Corollary 7.46]{o1983semi}.}
\label{thm:riemannian_submersion_props}
    Let $\pi: M \to N$ be a Riemannian submersion between Riemannian manifolds $(M, g)$
    and $(N, h)$ equipped with the Levi-Civita connection.
  If $\overline{\gamma}: I \to M$ is a geodesic that starts
      horizontally,~i.e. $\overline{\gamma}^{\prime}(0)$ is a horizontal vector,
      then $\overline{\gamma}$ is a \textbf{horizontal geodesic}
      ($\overline{\gamma}^{\prime}(t)$ is horizontal for all $t \in I$).
      Furthermore, the projection $\pi \circ \overline{\gamma} = \gamma$ is a geodesic
      in $N$ of the same length as $\overline{\gamma}$.
      Conversely, for any $p \in M$, if $\gamma$ is a geodesic in $N$
      with $\gamma(0) = \pi(p)$, there exists a unique local horizontal lift
      $\overline{\gamma}$ of $\gamma$ such that $\overline{\gamma}(0) = p$ and
      $\overline{\gamma}$ is a geodesic in $M$.
\end{thm}
Theorem~\ref{thm:riemannian_submersion_props} states that 
if a Riemannian submersion $\pi: M \to N$ is 
available, horizontal geodesics in $M$ are mapped to geodesics in $N$.
As $d\pi$ is an isomorphism when restricted to horizontal spaces, 
if we have a smooth cross-section $\sigma: N \to M$ we can 
express the Riemannian exponential on $N$ using the Riemannian exponential on $M$:
\begin{align}
  \exp_x(v) = \pi \circ \exp_{\sigma(x)}(\overline{v}),\quad \forall x \in N, v \in T_x N
\end{align}
}
\subsubsection{Naturally reductive \& Symmetric spaces}
\begin{defn}\footnote{\cite[Definition 23.8]{gallier2020differential}.}\label{def:reductive_homogeneous_space}
  Let $G$ be a Lie group, $H \leq G$ a closed subgroup $\Ad: G \to \genlin{\mathfrak{g}}$ be the adjoint representation of $G$.
  A homogeneous space $\text{G} / \text{H}$ is \textbf{reductive} if there is a
    subspace $\mathfrak{m}$ of $\mathfrak{g}$ where:
    \begin{align}
      \mathfrak{g} = \mathfrak{h} \oplus \mathfrak{m},~\text{and}~\Ad_h(\mathfrak{m}) \subseteq \mathfrak{m},~\forall h \in \text{H}
    \end{align}
    That is, $G / H$ is reductive if we 
    can find an $\Ad(H)$-invariant 
    subspace $\mathfrak{m}$ complementary to $\mathfrak{h}$ in $\mathfrak{g}$.
\end{defn}
The following property gives a recipe for constructing a $G$-invariant metric 
on $G/H$ and extending it to a left-invariant metric on $G$ that is right 
$H$-invariant such that $\pi: G \to G/H$ is a Riemannian submersion, 
with $\mathfrak{h}$ and $\mathfrak{m}$ being the vertical and horizontal 
subspaces at $e \in G$. 
 \begin{prop}\footnote{Corresponds 
   to \cite[Propositions 23.23]{gallier2020differential}.}\label{prop:reductive_hom_props}
  Let $G$ be a Lie group, $H$ a closed subgroup and 
  $G / H$ a reductive homogeneous 
  space with reductive decomposition $\mathfrak{g} = \mathfrak{h} \oplus \mathfrak{m}$.
  \begin{enumerate}
  \item There is a one-to-one correspondence
    between $G$-invariant metrics on $G / H$ and $\Ad(H)$-invariant 
    inner products on $\mathfrak{m}$. 
    The correspondence can be established by making $d\pi_e|_{\mathfrak{m}}: \mathfrak{m} \to T_{o}(G/H)$
    into a linear isometry, where $o = \pi(e) = eH$.
    A $G$-invariant metric on $G/H$ exists iff the closure of 
    $\Ad(H)(\mathfrak{m})$ is compact.
    If $H$ is compact then $\Ad(H)(\mathfrak{m})$ is compact, so there exists a
    $G$-invariant metric on $G / H$.

  \item Let $\mathfrak{m}$ have an $\Ad(H)$-invariant inner product. 
    If we extend it to an inner product on $\mathfrak{g} = \mathfrak{h} \oplus \mathfrak{m}$ 
    such that $\mathfrak{h}^{\bot} = \mathfrak{m}$,
    and endow $G$ with the corresponding left-invariant metric then the canonical map
    $\pi: G \to G/H$ is a Riemannian submersion.
  \end{enumerate}
\end{prop}
The reductive homogeneous spaces of interest are the following.
\begin{defn}\footnote{\cite[Definition 23.9]{gallier2020differential}.}\label{def:nat_reductive}
    Let $G$ be a Lie group and $H$ a closed subgroup of $G$.
    The homogeneous space $G / H$ is \textbf{naturally reductive} if it is
    reductive with decomposition $\mathfrak{g} = \mathfrak{h} \oplus \mathfrak{m}$, 
    has a $G$-invariant metric and satisfies: 
    \begin{align}
        \langle [X, Z]_{\mathfrak{m}}, Y \rangle = \langle X, [Z, Y]_{\mathfrak{m}} \rangle,~\forall X, Y, Z \in \mathfrak{m}
    \end{align}
\end{defn}
In this case, it is possible to express geodesics in $G / H$ with respect 
to the Levi-Civita connection as orbits of one-parameter subgroups generated 
by the tangent vectors in $\mathfrak{m}$.
\begin{prop}\footnote{\cite[Propositions 23.25-23.27]{gallier2020differential}.}\label{prop:geodesics_natred_homog}
  Suppose $G / H$ is a naturally reductive homogeneous space, 
  and we have $\mathfrak{g} = \mathfrak{h} \oplus \mathfrak{m}$.
  Using the $G$-invariant metric of $G/H$, a left-invariant metric is 
  constructed on $G$, such that its restriction to on $\mathfrak{m}$ 
  is $\Ad(H)$-invariant and we have $\mathfrak{m} = \mathfrak{h}^{\bot}$,
  and recall that in this case $\pi: G \to G/H$ is a Riemannian submersion.
  For every $X \in \mathfrak{m}$ the geodesic starting at $o = \pi(e) = eH$ with 
  initial velocity $d\pi_e(X)$ is given by:
\begin{align}\label{eq:geodesic_commuting_identity}
  \gamma_{o, d\pi_e(X)}(t) = \explie(tX) \cdot o = \pi \circ \explie(tX),\quad \forall t \in \sR
\end{align}
\end{prop}
Since the one-parameter subgroups $t \mapsto \explie(tX)$ are defined for any $t \in \sR$, 
by the preceding proposition so are maximal geodesics through $o$ and therefore 
through any point since we are working with a homogeneous space.
Naturally reductive homogeneous spaces 
are therefore complete\footnote{\cite[pp. 313]{o1983semi}.}.
\begin{defn}\label{def:riemannian_symmetric_space}
   A connected Riemannian manifold $(M, g)$ is a \textbf{(Riemannian) symmetric space}
   if for every point $p \in M$ there exists a unique isometry $s_p: M \to M$
   such that $s_p(p) = p$ and $(ds_p)_p = -\id_p$.
   Equivalently, for every $p \in M$, the map $s_p$ is an
   involutive isometry ($s^{2}_p = \id$) having $p$ as its only fixed point.
\end{defn}
The isometry $s_p: M \to M$ is called the \emph{global symmetry} of $M$ at the $p$.
Symmetric spaces can be constructed from 'Lie group data'. 
The connection can be made clear with a few more definitions.

An \textbf{involutive automorphism} of a Lie group $G$ is
an automorphism $\sigma: G \to G$ such that $\sigma \neq \id$
and $\sigma^{2} = \id$.
For $\sigma$ an involutive automorphism $G$, 
$G^{\sigma} = \{g \in G \mid \sigma(g) = g\}$ will denote the 
closed subgroup of fixed points of $\sigma$ and $G^{\sigma}_{0}$
its identity component.

\begin{defn}\footnote{\cite[Chapter IV, $\S3$]{helgason2001differential}.}
   \label{def:symmetric_pair}
   A \textbf{symmetric pair} is a triplet $(G, H, \sigma)$ where $G$ is a connected Lie group,
  $H$ a closed Lie subgroup of $G$, and $\sigma: G \to G$ an involutive
  automorphism of $G$ such that $G^{\sigma}_0 \subseteq H \subseteq G^{\sigma}$.
  If additionally $\Ad(H) \subseteq \genlin{\mathfrak{g}}$ is compact (where 
  $\Ad: G \to \genlin{\mathfrak{g}}$ is the adjoint representation of $G$), 
  then $(G, H, \sigma)$ is a \textbf{Riemannian symmetric pair}.
\end{defn}
The differential $d\sigma_e: \mathfrak{g} \to \mathfrak{g}$ 
of an involutive automorphism $\sigma: G \to G$ defines the $\pm 1$ 
eigenspaces:
\begin{align}\label{eq:cartan_decomp_eigenspaces}
  \mathfrak{h} = \{X \in \mathfrak{g} \mid d\sigma_e(X) = X\},\quad \mathfrak{m} = \{X \in \mathfrak{g} \mid d\sigma_e(X) = -X\}
\end{align}

\begin{thm}\footnote{\cite[Theorem 23.33]{gallier2020differential}.}
  \label{thm:cartan_decomp_theorem}
  Suppose that $(G, H, \sigma)$ is a symmetric pair.
  Then the following properties hold. 
  Note that items 1-3 make $G/H$ into a reductive homogeneous space.
  \begin{enumerate}
    \item $\mathfrak{h} = \{X \in \mathfrak{g} \mid d\sigma_e(X) = X\}$ 
      is the Lie algebra of $H$.
    \item $\mathfrak{g} = \mathfrak{h} \oplus \mathfrak{m}$, where 
      $\mathfrak{m} = \{X \in \mathfrak{g} \mid d\sigma_e(X) = -X\}$.
      The decomposition follows from $d\sigma_e: \mathfrak{g} \to \mathfrak{g}$ also being an involution $d\sigma_e^{2} = \id$ and the identity:
      \begin{align}\label{eq:lie_algebra_involution_decomp}
        X = \frac{1}{2} (X + d\sigma_e(X)) + \frac{1}{2}(X - d\sigma_e(X)),\quad \forall X \in \mathfrak{g}
      \end{align}
    \item $\Ad_k(\mathfrak{m}) \subseteq \mathfrak{m},~\forall k \in K$.
    \item $
    [\mathfrak{h}, \mathfrak{h}] \subseteq \mathfrak{h},
      [\mathfrak{h}, \mathfrak{m}] 
      \subseteq \mathfrak{m},~ [\mathfrak{m},\mathfrak{m}] \subseteq \mathfrak{h}
  $.
  \end{enumerate}
\end{thm}
The map $d\sigma_e: \mathfrak{g} \to \mathfrak{g}$ associated 
to a symmetric pair $(G, H, \sigma)$ is referred to as a \textbf{Cartan involution},
with the automorphism $\sigma: G \to G$ being a 
\textbf{\emph{global} Cartan involution}.
The decomposition $\mathfrak{g} = \mathfrak{h} \oplus \mathfrak{m}$ given by $d\pi_e$ 
as in Thm.~\ref{thm:cartan_decomp_theorem} is called a \textbf{Cartan Decomposition} 
of $\mathfrak{g}$.
If one further assumes that $G^{\sigma}_0$ and $H$ are compact, then 
we obtain a symmetric space.
\begin{thm}\footnote{See~\cite[Proposition 6.25]{ziller2010lie}  
  or~\cite[Proposition 3.4]{helgason2001differential}. Corresponds to~\cite[Theorem 23.34]{gallier2020differential},~\cite[Theorem 11.29]{o1983semi}.}
Suppose that $(G, H, \sigma)$ is a Riemannian symmetric pair 
with $G^{\sigma}_0$ and $H$ compact.
Denote by $\mathfrak{g}$ and $\mathfrak{h}$ the Lie algebras of $G$ and $H$ 
respectively.
\begin{enumerate}
  \item Since $H$ is compact, $G/H$ admits a $G$-invariant  
metric from Proposition~\ref{prop:reductive_hom_props} (1).
From the previous theorem, $G/H$ has a reductive decomposition 
  $\mathfrak{g} = \mathfrak{h} \oplus \mathfrak{m}$ where $\mathfrak{h}$ 
  and $\mathfrak{m}$ are the $\pm 1$ eigenspaces of $d\sigma_e$.
  Using the identity $[\mathfrak{m}, \mathfrak{m}] \subseteq \mathfrak{h}$ and 
  assuming a $G$-invariant metric on $G/H$ the natural reductivity condition 
  of~\ref{def:nat_reductive} holds 
  trivially (since $\mathfrak{h} \cap \mathfrak{m} = \{0\}$).
\item For every $p \in G/H$, there exists a isometry 
      $s_p: G/ H \to G/H$ such that $s_p(p) = p$ and $d(s_p)_p = -\textnormal{id}_p$, 
      making $G/H$ a Riemannian symmetric space. 
      For the projection $\pi: G \to G/ H$ and $o = \pi(e) = eH$, 
      the symmetry at $o$ is defined such that $s_o: gH \mapsto \sigma(g)H$:
      \begin{align}
      s_o \circ \pi = \pi \circ \sigma
      \end{align}
      For an arbitrary $p = gH \in G/H$, the geodesic symmetry is given by:
      \begin{align}
        s_p = \tau_g \circ s_o \circ \tau_{g^{-1}}
      \end{align}
\end{enumerate}
\end{thm}
By the preceding theorem symmetric spaces can be given a naturally reductive 
structure.
We can now use Proposition~\ref{prop:global_cross_section} to construct a global 
cross section, under which the Lie group $G$ can be identified with the 
product space $\mathfrak{m} \times H$ or $\explie(\mathfrak{m}) \times H$. 
Recall from Proposition~\ref{prop:geodesics_natred_homog} that geodesics starting 
at $o = \pi(e) = eH = H$ with initial velocity $d\pi_{e}(X)$ for $X \in \mathfrak{m}$ 
are of the form 
$\gamma_{o, d\pi_e(X)}(t) = \explie(tX) \cdot o = \pi(\explie(tX))$.
In particular, we can obtain the following expression for the 
Riemannian exponential $\exp_{o}: T_{o}M \to M$:
\begin{align}\label{eq:riemannian_exponential_commutes}
  \exp_{o}(d\pi_{e}|_{\mathfrak{m}}(X)) = \pi(\explie(X)), \quad \forall X \in \mathfrak{m}
\end{align}
That is, the following diagram commutes:
\begin{equation}\label{diag:orbit_riem_exponential_natreductive}
\adjustbox{scale=1.2}{\begin{tikzcd}
    \mathfrak{m} \arrow{r}{d\pi_{e}|_{\mathfrak{m}}} \arrow[swap]{d}{\explie} & T_{o}M \arrow{d}{\exp_{o}} \\
G \arrow{r}{\pi} & M
\end{tikzcd}
}
\end{equation}

\begin{prop}\label{prop:concrete_natreductive_cross_section}
    Let $\text{M} = \text{G} / \text{H}$ be a naturally reductive homogeneous space 
    and $\pi: G \to G/H$ the canonical projection.
    If the Riemannian exponential $\exp_{o}$ at the point 
    $o = \pi(e) = eH \in \text{M}$ is 
    a diffeomorphism, we can construct a diffeomorphism
    of $\mathfrak{m} \times H$ onto $G$ given by:
    \begin{align}\label{eq:bimeasurable_forward_liealgebra}
      \varPhi: \mathfrak{m} \times H \to G,\quad (X,h) \mapsto \explie(X)h
    \end{align}
    \begin{proof}
    \citet[Lemma 11.27]{o1983semi}.
    Again, we reproduce the proof as the maps defined are referenced in later sections.
    The map is built by constructing a cross-section of $\pi: G \to G/H$ 
    using the relation (\ref{eq:riemannian_exponential_commutes}) 
    of the Riemannian exponential such that one first defines:
    \begin{align}\label{eq:riem_exp_homogeneous_identity_group}
      \rExp_{e} \coloneqq \exp_{o} \circ~d\pi_{e}|_{\mathfrak{m}} = \pi \circ \explie: \mathfrak{m} \to M
    \end{align}
    By hypothesis $\exp_{o}: T_{o}(M) \to M$ is a diffeomorphism, and so 
    is $d\pi_{e}|_{\mathfrak{m}}$ making $\rExp_{e}$ a diffeomorphism. 
    We can define the cross-section by $\sigma: M \to G$ by:
    \begin{align}\label{eq:natreductive_cross_section}
    \sigma \coloneqq \explie \circ~\rExp_{e}^{-1}
    \end{align}
    Then $\pi \circ \sigma = \pi \circ \explie \circ~\rExp_{e}^{-1} = \rExp_{e} \circ \rExp^{-1}_{e} = \id_{M}$, and 
    by Proposition~\ref{prop:global_cross_section} we have a
    diffeomorphism $\varphi: M \times H \to G$ 
    given by (\ref{eq:bimeasurable_forward}): 
    \begin{align}\label{eq:bimeasurable_forward2}
      \varphi:(m, h) \mapsto \sigma(m)h = \explie(\rExp^{-1}_{e}(X))h
    \end{align}
    Composing this map with the map $\rExp_{e} \times \id_{H}$ we obtain the desired 
    map $\varPhi \coloneqq \varphi \circ (\rExp_{e} \times \id_{H})$:
    \begin{align}\label{eq:cartan_diffeomorphism_abstract}
      &\varPhi: \mathfrak{m} \times H \to G,\quad \varPhi: (X, h) \mapsto \sigma(\rExp_{e}(X))h = \explie(X)h
    \end{align}
    \end{proof}
\end{prop}

\subsection{The Cartan/Polar decomposition}\label{subsec:cartanpolar_appendix}
Define the following subsets of $\gmatrix{n}$:
\begin{align}
  \label{eq:symmetric_matrix} \gsym{n} &= \{P \in \gmatrix{n} \mid P = P^{T}\}\\
  \label{eq:spd_matrix} \gspd{n} &= \{P \in \gsym{n} \mid~\forall v \in \sR^{n}, v \neq 0,
      v^{T} P v > 0\}\\
    \label{eq:sspd_matrix}\gsspd{n} &= \{P \in \gspd{n}\mid \det(P) = 1\}\\
    \label{eq:ssym_matrix}\gssym{n} &= \{P \in \gsym{n}\mid \textnormal{tr}(P) = 0\}
\end{align}
$\gsym{n}$ is the vector space of $n \times n$ real symmetric matrices and
$\gspd{n}$ is the subset of $\gsym{n}$ of symmetric positive definite (SPD) matrices.
$\gsspd{n}$ denotes the subset of $\gspd{n}$ consisting
of SPD matrices with unit determinant, and $\gssym{n}$ the subspace of 
$\gsym{n}$ of traceless real symmetric matrices.
Every SPD matrix $S \in \gspd{n}$ has a 
unique square root\footnote{\cite[Lemma 2.18]{hall2015lie} or \cite[Theorem 7.2.6]{horn2012matrix}.} $S^{1/2} = P$, $P \in \gspd{n}$, which shows the uniqueness of the polar decomposition.
\begin{prop}[Polar decomposition]\footnote{
  See~\cite[Lemma 7.5.1]{jost2008riemannian} or~\cite[Theorem 7.3.1]{horn2012matrix}.}\label{prop:polar_decom_matrices}
    Any matrix $A \in \gglfull{n}$ can be uniquely decomposed as
    $A = PR$ or $A = \tilde{R}\tilde{P}$,
    where $P, \tilde{P} \in \gspd{n}$ and $R, \tilde{R} \in \go{n}$.
    We refer to the factorization $A = PR$ as the \textbf{left polar decomposition}
    and to $A = \tilde{R}\tilde{P}$ as the \textbf{right polar decomposition}.
    We choose to work with the left polar decomposition.
    The factors of this decomposition are uniquely determined and 
    we have a bijection $\gglfull{n} \to \gspd{n} \times \go{n}$ given by:
\begin{align}\label{eq:left_polar_map_appendix}
  A \mapsto (\sqrt{AA^{T}}, \sqrt{AA^{T}}^{-1}A),\quad \forall A \in \gglfull{n}
\end{align}
\end{prop}
As mentioned in the main text, this decomposition can be generalized using the fact that the spaces $\gspd{n} = \ggl{n} / \gso{n}$ and 
$\gsspd{n} = \gsl{n} / \gso{n}$ are symmetric spaces,
 and a \textbf{Cartan decomposition} is available in this case.
  We first state some useful properties of $\gspd{n}$ 
 and then  review its symmetric space and 
naturally reductive structure.
\begin{prop}\footnote{\cite[Theorems 2.6, 2.8 \& Corollary 2.9]{arsigny2007geometric} more specifically for the matrix exp/log and power map. For a review of algebraic properties and the Riemannian manifold structure of $\gspd{n}$ see~\cite{arsigny2007geometric,pennec2020manifold}.}\label{prop:spd_properties}
Every real symmetric matrix $X \in \gsym{n}$ has a spectral decomposition 
$X = ODO^{T}$ where $O \in \gso{n}$ and $D = \diag{d_{1}, \ldots, d_{n}}$, 
$d_{i} \in \sR$ is a diagonal matrix consisting 
of the eigenvalues of $X$, which are positive iff $X$ is positive-definite.
Using this decomposition we have simplified expressions 
for the matrix exponential $\explie: \gsym{n} \to \gspd{n}$ 
and logarithm $\logm: \gspd{n} \to \gsym{n}$:
\begin{align}
  \label{eq:symmetrix_matrix_exp} \explie(X) &= O\diag{\exp(d_1),\ldots,\explie(d_{n})}O^{T}, \quad \forall X \in \gsym{n}\\
  \label{eq:spd_matrix_log}\logm(P) &= O\diag{\log(d_1),\ldots,\log(d_{n})}O^{T}, \quad \forall P \in \gspd{n}
\end{align}
$\gspd{n}$ is an open subset of $\gsym{n}$ and a smooth manifold
of dimension $n(n + 1)/2$, with the tangent space 
$T_{P}\gspd{n}$ at any $P \in T_{P}\gspd{n}$ naturally isomorphic (by translation)
to $\gsym{n}$. 
The matrix exponential and logarithm maps are diffeomorphisms 
between $\gsym{n}$ and $\gspd{n}$, and the power map $P \mapsto P^{\alpha}$ is smooth 
for any $\alpha \in \sR$, since it can be expressed as:
\begin{align}\label{eq:spd_power_logexp_appendix}
  P^{\alpha} = \explie(\alpha\logm(P)),\quad \forall P \in \gspd{n}
\end{align}
\end{prop}
As a reference for the following results on $\gspd{n}$ 
and $\gsspd{n}$ see~\cite{forstner2003metric,pennec2020manifold,stegemeyer2021endpoint}. 
The presentation here 
also follows~\cite[Section 3.5]{rentmeesters2013algorithms} and~\cite[Section 3.5.3]{lezcano2022geometric}.

$\gspd{n}$ is a homogeneous space of the positive general 
linear group $\ggl{n}$.
More precisely, $\ggl{n}$ has a smooth transitive action on $\gspd{n}$ given by: 
\begin{align}\label{eq:gln_spd_action}
  \lambda: \ggl{n} \times \gspd{n} \to \gspd{n},\quad (A, P) \mapsto APA^{T}
\end{align}
Note that every SPD matrix $P \in \gspd{n}$ can be written 
as $P = AA^{T}$ for some $A \in \ggl{n}$.
The isotropy group of the identity matrix $I \in \gspd{n}$ corresponding 
to this action is the special orthogonal group $\gso{n}$ 
since $RIR^{T} = RR^{T} = I$ for $R \in \gso{n}$.

Applying Theorems~\ref{thm:homogeneous_space_maniffold} \&~\ref{thm:equivariant_diffeomorphism} 
we have that $\ggl{n} / \gso{n} = \gspd{n}$. That is, we have a diffeomorphism:
\begin{align}\label{eq:spd_quotient_diffeomorphism}
  \phi_I: \ggl{n} / \gso{n} \to \gspd{n},\quad A\cdot SO(n) \mapsto AA^{T}
\end{align}
And a smooth submersion of $\ggl{n}$ onto $\gspd{n}$ given by:
\begin{align}\label{eq:spd_projection_submersion}
  \pi_{I} = \phi_I \circ \pi: \ggl{n} \to \gspd{n},\quad A \mapsto AA^{T}
\end{align}

Let $\mathfrak{g} = \lggl{n}$ denote the Lie algebra of $\ggl{n}$.
The Lie algebra of $\gso{n}$ is the space 
$\lgso{n} = \{X \in \gmatrix{n} \mid X = -X^{T}\}$ of skew-symmetric matrices.
$\gspd{n} = \ggl{n} / \gso{n}$ is a reductive homogeneous 
space (see Definition~\ref{def:reductive_homogeneous_space}) since $\Ad(\gso{n})(\gsym{n}) \subseteq \gsym{n}$ and we have the decomposition $\mathfrak{g} = \mathfrak{h} \oplus \mathfrak{m}$ given by:
\begin{align}\label{eq:symmetric_antisymmetric_decomp}
  \lggl{n} = \lgso{n} \oplus \gsym{n}
\end{align}
Then $\mathfrak{h} = \lgso{n}$ and $\mathfrak{m} = \gsym{n}$,
and the bracket relations $[\mathfrak{h}, \mathfrak{h}] \subseteq \mathfrak{h}, [\mathfrak{h}, \mathfrak{m}] \subseteq \mathfrak{m},~ [\mathfrak{m},\mathfrak{m}] \subseteq \mathfrak{h}$ hold.
We now choose an inner product on $\lggl{n}$ such that 
its restriction to $\gsym{n}$ is $\Ad(\gso{n})$-invariant, $\gso{n}^{\bot} = \gsym{n}$
and we can use it to define a left-invariant Riemannian metric on $\ggl{n}$.
We work with a scaled version of the canonical inner product 
$\langle X, Y \rangle \coloneqq \textnormal{tr}(X^{T}Y)$:
\begin{align}\label{eq:canonical_inner_product_scaled}
  B(X, Y) \coloneqq 4\langle X, Y \rangle = 4\textnormal{tr}(X^{T}Y),\quad X, Y \in \lggl{n}
\end{align}
The inner product respects the decomposition (\ref{eq:symmetric_antisymmetric_decomp}) 
into symmetric and skew-symmetric matrices, and the 
left-invariant metric on $\ggl{n}$ (which is also right-$\gso{n}$-invariant) is: 
\begin{align}\label{eq:gln_metric}
  g^{\ggl{n}}_{A}(X, Y) = B(A^{-1}X, A^{-1}Y), \quad \forall A \in \ggl{n},~\forall X,Y \in T_{A}\ggl{n}
\end{align}
To define a $\ggl{n}$-invariant metric on $\gspd{n} = \ggl{n}/\gso{n}$, note that 
the differential of the projection (\ref{eq:spd_projection_submersion}) at $I$ is 
$d\pi_{I}(X) = X + X^{T}$ for any $X \in \lggl{n}$, with $\ker(d\pi_{I}) = \lgso{n}$
and its restriction to $\mathfrak{m} = \gsym{n}$ gives the isomorphism:
\begin{align}\label{eq:projecton_differential_forward}
  d\pi_{I}: \mathfrak{m} \to T_{I}\gspd{n},\quad X \mapsto 2X
\end{align}
We have $\lambda(P^{1/2}, I) = P$, and the differential 
with respect to the second argument at identity is $X \mapsto P^{1/2}XP^{1/2}$. 
The linear isomorphism $d(\pi_{I} \circ L_{P^{1/2}})_{I}: \mathfrak{m} \to T_{P}\gspd{n}$ 
is then given by: 
\begin{align}\label{eq:forward_differential_projection}
  d(\pi_{I} \circ L_{P^{1/2}})_{I}: X \mapsto 2P^{1/2}XP^{1/2},\quad \forall X \in \mathfrak{m}
\end{align}
We denote its inverse by $\eta_{P}: T_{P}\gspd{n} \to \mathfrak{m}$, 
such that $\eta_{P}: X \mapsto \frac{1}{2}P^{-1/2}XP^{-1/2}$.
The induced (quotient) metric on $\gspd{n}$\footnote{Known as the 
affine-invariant metric, see~\cite[Proposition 3.1]{thanwerdas2023n}.} is defined for any $P \in \gspd{n}$ and 
$X, Y \in T_{P}\gspd{n}$:
\begin{align}\label{eq:spd_metric}
  g^{\gspd{n}}_{P}(X, Y) \coloneqq B(\eta_{P}(X), \eta_{P}(Y)) = \langle P^{-1/2}XP^{-1/2}, P^{-1/2}XP^{-1/2} \rangle = \textnormal{tr}(P^{-1}XP^{-1}Y)
\end{align}
Endowed with this metric the action of $\ggl{n}$ is by isometries 
and $\pi_{I}$ is a Riemannian submersion.
$(\gspd{n}, g^{\gspd{n}})$ is also a Riemannian symmetric space and $(\ggl{n}, \gso{n}, \Theta)$ is a Riemannian symmetric pair, 
with $\Theta$ the global Cartan involution:
\begin{align}\label{eq:cartan_involution_gln}
  \Theta: \ggl{n} \to \ggl{n}, \quad \Theta: A \mapsto (A^{T})^{-1}
\end{align}
In this case we have $G^{\Theta}_{0} = \gso{n} = G^{\Theta}$, and 
we have corresponding Lie algebra involution:
\begin{align}\label{eq:lie_algebra_involution_gln}
  \theta \coloneqq d\Theta_e: \lggl{n} \to \lggl{n},\quad \theta: X \mapsto -X^{T}
\end{align}
Analog results hold for $\gsspd{n} = \gsl{n} / \gso{n}$, such that 
$(\gsl{n}, \gso{n}, \Theta)$ is a Riemannian symmetric pair.
 The group $\gsl{n}$ has Lie 
 algebra:
\begin{align}\label{eq:sln_lie_algebra}
  \lgsl{n} = \{X \in \lggl{n} \mid \textnormal{tr}(X) = 0\}
\end{align}
$\gsl{n} / \gso{n}$ is an example of a non-compact 
symmetric space\footnote{For a classification of symmetric spaces 
see for example~\cite[Chapter 6]{ziller2010lie}.}.
We can reuse the metrics (\ref{eq:gln_metric}) and (\ref{eq:spd_metric}), restricting
them to $\gsl{n}$ and $\gsspd{n}$, respectively.
\begin{prop}\footnote{A detailed treatment can be found in~\cite{dolcetti2014some, dolcetti2018differential}.
}
  \label{prop:sln_gln_relationship}
  $\ggl{n}$ can be represented as a product $\gsl{n} \times \sR^{\times}_{>0}$ by the Lie group isomorphism:
  \begin{align}\label{eq:gln_sln_map}
    \ggl{n} \to \gsl{n} \times \sR^{\times}_{>0},\quad A \mapsto (\frac{A}{\det(A)^{\frac{1}{n}}}, \det(A)^{\frac{1}{n}})
  \end{align}
 Reusing the previously defined metrics, 
 $\gsl{n}$ and $\gsspd{n}$ are totally geodesic 
 submanifolds\footnote{$N$ is a totally geodesic submanifold of a Riemannian manifold $M$, if for any two points of $N$ and a geodesic $\gamma$ in $M$ that joins them, $\gamma$ is entirely contained in $N$.
 } of $\ggl{n}$ and $\gspd{n}$, respectively.
 The decomposition (\ref{eq:gln_sln_map}) restricted to $\gspd{n}$ 
can be shown to induce a Riemannian isometry 
$\gspd{n} \cong \gsspd{n} \times \sR_{>0}$.
The tangent space decomposition is $\gsym{n} = \gssym{n} \oplus \mathfrak{d}$, where $\mathfrak{d}(n, \sR)$ are scalar diagonal matrices.
\end{prop}
\subsection{Proof of theorem~\ref{thm:manifold_splitting_theorem}}\label{subsec:manif_split_thm_proof}
As in the main text, we let $(G/H, M, \mathfrak{m})$ define our `Lie group data', 
corresponding to $(\ggl{n}/\gso{n}, \gspd{n}, \gsym{n})$ 
or $(\gsl{n}/\gso{n}, \gsspd{n}, \gssym{n})$.
\ManifoldSplittingThm*
\begin{proof}
  The theorem is a collection of results related to the Cartan decomposition 
and the structure theory of Lie groups, which can be found for example 
in~\cite[Chapter II.10]{bridson2013metric} or~\cite[Chapter 6]{abbaspour2007basic}.
Similar results apply to algebraic subgroups 
of $\gglfull{n}$ that are closed and stable 
under transposition (see~\cite[Prop. 6.3.3 \& Definition 6.3.4]{abbaspour2007basic} 
or~\cite[Definition 10.56]{bridson2013metric}).
  \begin{enumerate}
    \item The first result holds due to Proposition~\ref{prop:spd_properties}, and the 
  fact that for any $X \in \mathfrak{m}$ we have $\explie(tX) \in M$ for all $t \in \sR$, see~\cite[Lemma 10.52]{bridson2013metric}.
  The Riemannian exponential on $M$ is also a diffeomorphism at any point.

  \item For the group-level Cartan/Polar decomposition see~\cite[Theorem 6.2.5 \& 6.3.5]{abbaspour2007basic}.
  Given a tangent vector in $\mathfrak{m}$, the Riemannian exponential on $M$ 
      and the matrix exponential are related by the diffeomorphism:
  \begin{align}\label{eq:riemannian_exponential_special}
    \rExp_{e}: \mathfrak{m} \to M,\quad X \mapsto \explie(X)\cdot I = \explie(X)I\explie(X)^{T} = \explie(2X)
  \end{align}
    $\rExp_{e}$ is obtained from applying Proposition~\ref{prop:geodesics_natred_homog}.
  The map $\varPhi$ of (\ref{eq:cartan_lie_algebra_diffeomorphism}) can be 
      obtained from (\ref{eq:cartan_decomposition_diffeomorphism}) 
      and the fact that the matrix exponential is a diffeomorphism on $M$, or 
      using Proposition~\ref{prop:concrete_natreductive_cross_section}, 
      such that (\ref{eq:cartan_lie_algebra_diffeomorphism}) 
      corresponds to (\ref{eq:cartan_diffeomorphism_abstract}).

      \item The map $\chi^{-1}$ is simply the polar decomposition.
      To obtain $\xi^{-1}$ we use the fact that ${AA^{T} \in M}$ for $A \in G$ and the identities:
      \begin{align}\label{eq:spd_square_root_properties}
        \logm(P^{1/2}) = \frac{1}{2}\logm(P),~P^{-1/2} = \explie(-\frac{1}{2}\logm(P)),\quad \forall P \in M
      \end{align}
      The identities (\ref{eq:spd_square_root_properties}) 
      can be obtained from (\ref{eq:spd_power_logexp}).
  \end{enumerate}
\end{proof}

\subsection{Integral factorizations for the Cartan/Polar decomposition}\label{subsec:appendix_polar_haar}
Consider again the notation $(G/H, M, \mathfrak{m})$ 
as in Theorem~\ref{thm:manifold_splitting_theorem}.
Recall that we have $G = \ggl{n}$ ($M = \gspd{n}$) or $G = \gsl{n}$ ($M = \gsspd{n}$), 
with $H = \gso{n}$.
From the proof of Thm.~\ref{thm:manifold_splitting_theorem} the cross-section $\sigma = \explie \circ \rExp^{-1}_{e}: M \to G$, of Prop.~\ref{prop:concrete_natreductive_cross_section} is:
\begin{align}\label{eq:concrete_cross_section}
\sigma(S) = \explie(\frac{1}{2}\logm(S)),\quad \forall S \in M
\end{align}
which is smooth (Prop.~\ref{prop:spd_properties}) and reduces simply to 
the square root $\sigma(S) = S^{1/2}$.
Since symmetric positive definite matrices have a unique square root, 
$\sigma: M \to G$ is a diffeomorphism.
We obtain a decomposition equivalent to the Polar/Cartan decomposition given 
by the map ${\varphi: M \times H \to G}$ defined 
as in Propositions~\ref{prop:global_cross_section}-\ref{prop:concrete_natreductive_cross_section}:
\begin{align}
  \label{eq:polar_decomp_haar_special}\varphi: M \times H \to G,\quad \varphi: (S, R) \mapsto \sigma(S)R = S^{1/2}R
\end{align}
With the inverse defined by:
\begin{align}
  \label{eq:inverse_polar_decomp_haar_special}\psi: G \to M \times H,\quad \psi: A \mapsto (\pi_{I}(A), (\sigma(\pi_{I}(A)))^{-1}A) = (AA^{T}, (AA^{T})^{-1/2}A)
\end{align}
We have equivalent decompositions which allow us to represent 
$A \in G$ as $A = PR$ 
or $A = S^{1/2}R$ for $S, P \in M$, $R \in H$ and therefore $P = S^{1/2}$.
The motivation behind presenting both decompositions is that 
for $\gglfull{n}$, the decomposition $A = S^{1/2}R$, 
has a factorization of the Haar measure $\mu_{\gglfull{n}}$ as a product 
of invariant measures on $\gspd{n}$ and $\go{n}$.
The Haar measure on $\gglfull{n}$ is given for any $A = (A_{ij}) \in \gglfull{n}$ 
by\footnote{\cite[VII, $\S3$, No. 3, Example 1]{bourbaki2004integration}.}:
\begin{align}\label{eq:gln_haar_measure}
  d\mu_{\gglfull{n}}(A) = \lvert \det(A) \rvert^{-n}dA = \lvert \det(A) \rvert^{-n} \prod_{i, j = 1}^{n} dA_{ij}
\end{align}
where $dA$ is the Lebesgue measure on $\sR^{n^2}$ and $dA_{ij}$ 
is the Lebesgue measure on $\sR$.
$\gglfull{n}$ has two homeomorphic connected components consisting of 
the group of invertible matrices with positive determinant
$\ggl{n}$ and with negative 
determinant $\gglminus{n}$\footnote{\cite[Theorem 3.68]{warner1983foundations}.}.
Integrating the full group $\ggl{n}$ can be done by integrating each component separately, 
and we focus on constructing a solution for the identity component $\ggl{n}$.
We use the shorter notation $\shgspd{n}$ and $\shgsspd{n}$ going forward to denote $\gspd{n}$ 
and $\gsspd{n}$.
Using a similar notation scheme as in (\ref{eq:gln_haar_measure}), 
the unique (up to scaling) $\gglfull{n}$-invariant measure 
on $\shgspd{n}$ 
is\footnote{\cite[VII, $\S3$, No. 3, Example 8]{bourbaki2004integration}. 
$\mu_{\gspd{n}}$ is bi-$\gglfull{n}$-invariant, as well as invariant 
under $P \mapsto P^{-1}$.}:
\begin{align}\label{eq:spd_haar_measure}
  d\mu_{\shgspd{n}}(S) = \lvert \det(S) \rvert^{-(n+1)/2} dS  = \lvert \det(S) \rvert^{-(n + 1)/2} \prod_{1 \leq i \leq j \leq n}^{n} dS_{ij},\quad \forall S \in \shgspd{n}
\end{align}
The following result can be found in a more general setting, 
often expressed using the `right' polar coordinates of 
the decomposition $A = RS^{1/2}$.
Let $H = \gV_{n, m} = \{R \in \gmatrix{nm}\ \mid R^{T}R = I_m\}$ for $ n \geq m$ and
$G = \gmatrix{nm}^{\ast} = \{A \in \gmatrix{nn} \mid \textnormal{rank}(A) = m\}$ 
the set of $n \times m$ matrices of rank $m$.
$\gV_{n, m}$ is the Stiefel manifold of orthonormal $m$-frames in $\sR^{n}$, 
on which $\go{n}$ acts transitively by left multiplication 
such that $\gV_{n, m} = \go{n} / \go{n - m}$, with 
special cases $\gV_{n, n} = \go{n}$ and $\gV_{n, n - 1} = \gso{n}$.
The complement of $\gmatrix{nm}^{\ast}$ in $\gmatrix{nm}$ has Lebesgue measure 
zero\footnote{\cite[Example 5.5]{gross1976bessel} 
or~\cite[Proposition 7.1]{eaton1983multivariate}.}, 
and $\gmatrix{nn}^{\ast} = \gglfull{n}$.
In this case it will correspond to~\cite[Lemma 1.4]{herz1955bessel} 
or~\cite[Theorem 2.1.14]{muirhead2009aspects}.
\PolarCoordinateChangeOfVar*
\begin{proof}
  A proof is given in~\cite[Prop. 5.6]{gross1976bessel} for 
  the decomposition of the form $A = RS^{1/2}$.
  In the form $S^{1/2}R$ it is proven for 
  example in~\cite[Section 4]{faraut1987bessel}.
  In the context of multivariate statistics 
  see Theorem 5.2.2 and Remark 5.2.3 of \cite{farrell2012multivariate}. 
  A recent reference is~\cite[Section 16.7.2]{chirikjian2012algebraic}.
\end{proof}
Note that the constant $\beta_{n}$ is independent of $f \in C_c(G)$.
From~\cite[(16.36)]{chirikjian2012algebraic}:
\begin{align}\label{eq:orthogonal_group_volume} 
 \textnormal{Vol}(\go{n}) = 2 \cdot \textnormal{Vol}(\gso{n}) = \frac{2^{n}\pi^{n^{2}/2}}{\Gamma_{n}(n / 2)}
\end{align}
Where $\Gamma_{n}(\cdot)$ denotes the multivariate Gamma function.
From~\cite[(16.55) \& (16.56)]{chirikjian2012algebraic}, 
if $dA$ is the Lebesgue measure on $\sR^{n^2}$, under the decomposition 
$A = S^{1/2}R$ we have:
\begin{align}\label{eq:lebesgue_polar_coordinates}
  dA = d(S^{1/2}R) = \beta_{n} \lvert \det(S) \rvert^{1/2} dO dS
\end{align}
Then considering that $S = AA^{T}$, 
the Haar measure $d\mu_{\gglfull{n}} = \lvert \det(A) \rvert^{-n} dA$ can 
be expressed:
\begin{align}
  d\mu_{\gglfull{n}}(dA) = \lvert S^{1/2}R \rvert^{-n} d(S^{1/2}R) = \beta_{n}~\lvert \det(AA^{T}) \rvert^{-n/2}  \lvert \det(S) \rvert^{1/2} dO dS
\end{align}

We use this decomposition treating $G$ ($\ggl{n}$ or $\gsl{n}$) as our sample space.
From Section 5.2 of~\cite{farrell2012multivariate}, for the case 
$G = \gmatrix{nm}^{\ast} \cong \gV_{n,m} \times \shgspd{m}$ it can be shown 
that if a $A \in \gmatrix{nm}^{\ast}$ is a random matrix with a
$\go{n}$-left invariant distribution, then for $\varphi(A) = RS^{1/2}$ 
the corresponding random variables $R \in \gV_{n, m}$ and $S^{1/2} \in \shgspd{m}$ 
will be independent, and $R$ will have a uniform distribution on $\gV_{n, m}$.
Furthermore, there exists a relationship 
between the density function of $A = RS^{1/2} \in \gmatrix{nm}^{\ast}$ with respect 
to the $\go{n}$-invariant measure and that 
of $S \in \shgspd{n}$ with respect to (\ref{eq:spd_haar_measure}).
If $G = \ggl{n}$ or $G = \gsl{n}$,
the Haar measure $\mu_{G}$ is 
bi-$\go{n}$-invariant (respectively bi-$\gso{n}$-invariant).
We can then work with either decomposition~\footnote{In the second case, one considers the right action $A^{T}PA$, 
for $A \in \gglfull{n}$ and $P \in \shgspd{n}$.} $S^{1/2}R$ or $RS^{1/2}$.

\GLnProductSamplingIndep*
\begin{proof}
  This theorem collects Lemma 5.2.4 \& 5.2.8 of~\cite{farrell2012multivariate} 
  applied to the case where we are working with random matrices in $\ggl{n}$ 
  and using the left polar decomposition.
See also~\cite[Proposition 7.4]{eaton1983multivariate}.
\end{proof}

Restricting only to the connected component $\ggl{n}$,  
the task is now to specify a probability distribution on 
$\shgspd{n}$ relative to the measure $d\mu_{\shgspd{n}}$.
For the case of $\gsl{n}$, using the isomorphism (\ref{eq:gln_sln_map}), 
we can define a $\gsl{n}$-invariant 
measure on $\shgsspd{n}$.
More precisely, $P = (\det(P)^{1/n}I)\tilde{P}$ 
for $P \in \shgspd{n}$, $\tilde{P} = \det(P)^{-1/n}P \in \shgsspd{n}$, which 
we write $P = t^{1/n}\tilde{P},~t>0$, such 
that $d\mu_{\shgspd{n}}(P) = \frac{dt}{t}d\mu_{\shgsspd{n}}(\tilde{P})$ and 
for $f \in L^{1}(\shgspd{n})$\footnote{\cite[(1.21) \& (1.39)]{terras2016harmonic}.}:
\begin{align}\label{eq:spd_sspd_haar_relation}
  \int_{\shgspd{n}} f(P) \dd{\mu_{\shgspd{n}}(P)} = \int_{t > 0} \int_{\shgsspd{n}}f(t^{1/n}\tilde{P})\frac{\dd{t}}{t}\dd{\mu_{\shgsspd{n}}(\tilde{P})}
\end{align}

\subsubsection{Sampling on the SPD manifold}\label{subsec:appendix_spd_distributions}

Following~\cite{said2017riemannian}, a Riemannian Gaussian Distribution denoted as $G(\overline{P}, \sigma)$
depends on parameters $\overline{P} \in \shgspd{n}$ and $\sigma > 0$ to
define a probability density function with respect to the volume element
$d\mu_{\shgspd{n}}$ by: \begin{align}\label{eq:said_density} p(P|\overline{P},
\sigma)=\frac{1}{Z(\sigma)} \exp[-\frac{d^2(P, \overline{P})}{2 \sigma^2}],~P
\in \shgspd{n} \end{align} Here, $d: \shgspd{n} \times \shgspd{n} \to \sR_{\geq
0}$ is the Riemannian distance corresponding to the affine-invariant metric (\ref{eq:spd_metric}).  The metric
plays a key role, as the measure $(\ref{eq:spd_haar_measure})$ is the Riemannian
volume element associated to it.  The distance can be expressed by:
\begin{align}\label{eq:spd_squared_distance} d^{2}(X, Y) =
\textnormal{tr}[\logm(X^{-1/2}YX^{-1/2})]^{2},\quad \forall X, Y \in \shgspd{n}
\end{align} and $Z(\sigma)$ is a normalization factor given
in~\cite{said2017riemannian} by:
\begin{align}\label{eq:spd_haar_normalization_factor}
\int_{\shgspd{n}}\exp[-\frac{d^2(P, \overline{P})}{2
\sigma^2}]\dd{\mu_{\shgspd{n}}(P)} \end{align} For $n=2$ an analytic expression
of $Z(\sigma)$ exists, otherwise it can be approximated by Monte Carlo
integration.  \cite[Prop. 5 \& 6]{said2017riemannian} describe an
algorithm for sampling from this distribution.

Alternatively, an approximate solution when sampling close to the identity is 
given by the Log-normal distribution 
defined in~\cite[Sec. 4.4]{schwartzman2016lognormal}.
From~\cite[Def. 4.4.2]{schwartzman2016lognormal}, $X \in \shgspd{n}$ 
has a Log-normal distribution $X \sim LN(M, \Sigma)$ with 
mean $M \in \shgspd{n}$ and 
covariance $\Sigma$ if $\logm(M^{-1/2}XM^{-1/2}) \sim N(0, \Sigma)$.
This definition assumes that $M$ is the empirical Riemannian center of mass, 
corresponding to the random variable $X$.

\revision{
\subsubsection{Alternative decomposition based on the QR factorization}\label{subsec:appendix_qr_fact}
}
There are several choices available for decomposing $\ggl{n}$ and $\gsl{n}$ such 
that invariant integration can be made easier while working with the smaller factors.
The primary tools of interest are the Iwasawa and the Cartan decomposition, and one 
possibility is given by the Gram decomposition (QR factorization). 
Let ${\gutp{n} = \{X \in \gglfull{n} \mid X_{ij} = 0\text{ if } i > j\}}$ 
be the group of real upper triangular matrices and $\gutp{n}_{+} \leq \gutp{n}$ its subgroup 
whose diagonal entries are positive.
Every matrix $A \in \gglfull{n}$ has a unique 
decomposition as $A = RT$ or $A = TR$ 
for $T \in \gutp{n}_{+}$ and $R \in \go{n}$.

Under this decomposition, Theorem~\ref{thm:decompose_haar_abstract} (2) is applicable.
The orthogonal factor becomes ${R \in \gso{n}}$ if restricted to 
$A \in \ggl{n}$.
For $A \in \gsl{n}$ the decomposition is given by replacing $\gutp{n}_{+}$ with its 
subgroup $\gsutp{n}_{+} \leq \gutp{n}_{+}$ of matrices with unit determinant.
\subsection{More details on the Lie algebra parametrization}\label{subsec:lie_algebra_param_appendix}
Any $A \in G$ can be expressed uniquely as $A = e^{X}R$ 
for $x \in \mathfrak{m}$ and $R \in H$.
Since $H = \gso{n}$ in both cases, the fact that $\explie: \lgso{n} \to \gso{n}$ is 
surjective\footnote{\cite[Theorem 2.6]{gallier2020differential}.}, allows us to write
it $A = e^{X}e^{Y}$, $Y \in \lgso{n}$.
The factors $X$ and $R = e^{Y}$ are obtained using $\varPhi^{-1}$ 
(\ref{eq:cartan_lie_algebra_inverse_diffeomorphism}).
Then by taking the principal branch of the matrix logarithm  
on $H = \gso{n}$, $Y = \logm(R)$.
A map $\xi^{-1}: G \to \mathfrak{g}$ as described in Section~\ref{sec:contribution} 
is therefore constructed as $\xi^{-1} = (\id_{\mathfrak{m}} \times~\logm) \circ \varPhi^{-1}$.
More precisely, for any $A = e^{X}e^{Y} \in G$, using $\xi^{-1}$ we obtain 
the horizontal/vertical tangent vectors 
$(Y, X) \in \lgso{n} \times \mathfrak{m}$ and since 
$\mathfrak{g} = \lgso{n} \oplus \mathfrak{m}$ we have a unique 
$Z = X + Y \in \mathfrak{g}$.

If $d$ is the dimension of $G$, the tangent space $\mathfrak{g}$ 
is a $d$-dimensional vector space isomorphic to $\sR^{d}$, with
basis elements denoted by $(E_1, \ldots, E_d)$.
Once a basis is chosen we can concretely 
represent any element of $\mathfrak{g}$ (or $\mathfrak{h}$, $\mathfrak{m}$) as a linear
combination of the `generators' 
such that $v = \sum_{i=1}^{d} v_i E_i$ for any $v \in \mathfrak{g}$.
The \emph{vee} and \emph{hat} functions (denoted $\vee$ and $\liehat$) 
are used to map tangent vectors to their coordinates 
in this basis and back:
\begin{align}
  \label{eq:hat}\liehat&: \sR^{d} \to \mathfrak{g},\quad \liehat: \ervv = (v_1, v_2, \ldots, v_d)^{T} \mapsto \ervv^{\liehat} = \sum_{i=1}^{k} v_i E_i\\
  \label{eq:vee}\vee&: \mathfrak{g} \to \sR^{d},\quad \vee: \ervv^{\liehat} \mapsto (\ervv^{\liehat})^{\vee} = \ervv
\end{align}

The basis $(E_{i})_{i \in [d]}$ is chosen to be orthonormal 
with respect to the inner product (\ref{eq:canonical_inner_product_scaled}) 
which is used to construct the invariant metric.
Going forward it is understood that functions parametrized 
on the Lie algebra, such as the kernel $\tilde{k}_{\theta}: \mathfrak{g} \to \sR$, 
take as input the vector of scalar 
coefficients of the tangent vector expressed 
in the chosen basis (the result of the $\vee$ map).

To summarize, the map $\xi^{-1}: G \to \mathfrak{g}$ is implemented 
for any $A \in G$ by\footnote{
Similar interpolation methods have been explored in the context 
of numerical linear algebra and computational mathematics.
See~\cite{munthe2001generalized, munthe2014symmetric} and 
especially~\cite{gawlik2018interpolation}, as this method can in part be 
seen as an instance of their more general framework. 
}:
\begin{enumerate}
  \item Mapping $A$ to its product space representation 
    in $\mathfrak{m} \times \gso{n}$ using $\varPhi^{-1}(A) = (X, R)$.
  \item Using the matrix logarithm on $R = e^{Y}$ (which is available in closed form 
    for the cases of interest $\gso{2}$ and $\gso{3}$) 
    to obtain $(X, \logm(R)) = (X, Y)$.
  \item Expressing the tangent vector $Z = X + Y$ using 
    the chosen basis as $Z^{\vee} \in \sR^{d}$.
\end{enumerate}

\section{Architecture \& training details}All experiments will use the 
same ResNet-like architecture~\cite{he2016deep}, and it will consist of 
a lifting cross-correlation layer, a single residual block and a final 
cross-correlation layer.
Finally, to achieve invariance global pooling is applied over 
the spatial and group dimensions.
The (lifting) cross-correlation layers are always 
followed by normalization and non-linear activation layers. 
In the case of the affine robustness task, we use GeLU nonlinearities 
and `LayerNorm' 
normalization\footnote{See~\cite{hendrycks2016gaussian} and~\cite{ba2016layer}.}. 
The residual block contains $2$ group cross-correlation layers 
and we apply max-pooling over the spatial dimension of the feature maps after 
each block to increase the robustness of the model.
For all experiments, the kernels ${k_{\theta}: \mathfrak{g} \to \sR}$ are 
parametrized using `SIREN networks', introduced in~\cite{sitzmann2020implicit}.
SIREN networks can be considered as one example of an Implicit Neural 
Representation (INR) model.
These models have seen widespread use 
in various areas of computer vision and graphics, e.g.~\cite{mildenhall2021nerf}.
INRs can be formalized as learned continuous 
function approximators based on MLPs.
They can be described simply as MLP layers of the form:
\begin{align}
    \mathbf{y}_m=\sigma\left(W_m \mathbf{y}_{m-1}+\mathbf{b}_m\right)
\end{align}
where $\sigma$ is a non-linearity. 
In case of SIRENs we have $\sigma(x) = \sin(\omega_{0}x)$, where $\omega_{0} \in \sR_{>0}$ is a 
multiplier controlling the frequency of the 
sinusoid.
We emphasize again that the proposed methodology is not dependent 
on the specific parametrization of $k_{\theta}$, and have experimentally 
found that other activation functions such 
as the (complex) Gabor wavelet~\cite{saragadam2023wire} offer comparable results.
We set $\omega_{0} = 10$ for all experiments.
We use $42$ output channels in both the lifting 
and cross-correlation layers.
Each SIREN network consists of $2$ layers of size $60$.

A key hyperparameter to consider is the number of group elements that will be sampled in the Monte 
Carlo approximation of each of the cross-correlation layers.
Empirically, we have found that $10-12$ samples are enough to 
achieve a better performance compared to the previously described models.
The models are trained for \revision{$100$} epochs, with a batch size of $128$, and the Adam optimizer of~\cite{kingma2014adam} with a standard learning rate of $0.0001$.
Sampling from $\gspd{2}$ is done using the log-Normal 
distribution of~\cite{schwartzman2016lognormal} centered at the identity 
while for $\gso{2}$ we work with a discretization of 
equi-distant points in $[0, 2\pi]$.

\section{Equivariance error analysis}\label{sec:equivariance_error_sec}

\revision{
\paragraph{Equivariance error}
Since our models are only equivariant in expectation, we validate this property
numerically by measuring their equivariance error following the same approach
as~\citet{Sosnovik2020Scale}, where we look to quantify for a neural network $\Phi$ and any $g \in G$ the relative error:
\begin{align}\label{eq:equivariance_error}
\Delta \coloneqq \|\gL_{g}[\Phi(f)] - \Phi[\gL_{g}(f)]\|^{2}_{2}/\|\gL_{g}[\Phi(f)]\|^{2}_{2}
\end{align}
We evaluate the equivariance error before training the network, i.e. $\Phi$ is a
convolutional network with randomly initialized weights.  We take $\Phi$ to be a
simple convolutional network composed of a lifting map
(\ref{eq:lifting_cross_correlation}), a cross correlation
(\ref{eq:conv_cross_ops2}) and a projection cross-correlation $C^{\downarrow}_{k}: L^{1}(G) \to L^{1}(X)$,
with $k: X \to \sR$, mapping our data back to the homogeneous space $X$ (where $X = \sR^{2}$ in this
case):
\begin{align}
  C^{\downarrow}_{k}: f \mapsto C^{\downarrow}_{k}f, \quad C^{\downarrow}_{k}f: x \mapsto \int_{G} f(\tilde{g}) k(\tilde{g}^{-1}x) \dd{\mu_G(\tilde{g})},~\forall x \in X
\end{align}
The same normalization and nonlinearities described previously are employed.
In Figure~\ref{fig:equivariance_error_sl2} we plot the equivariance error of $\Phi$, for different choices of $k_{\theta}$, and compare our model to a standard CNN with the same input-output dimensionality for its layers. We produce $100$ samples from the Haar measure of $\gsl{2}$, and obtain an average estimate over $10$ random seeds.
\begin{figure}[ht]
\centering
  \includegraphics[width=0.8\linewidth]{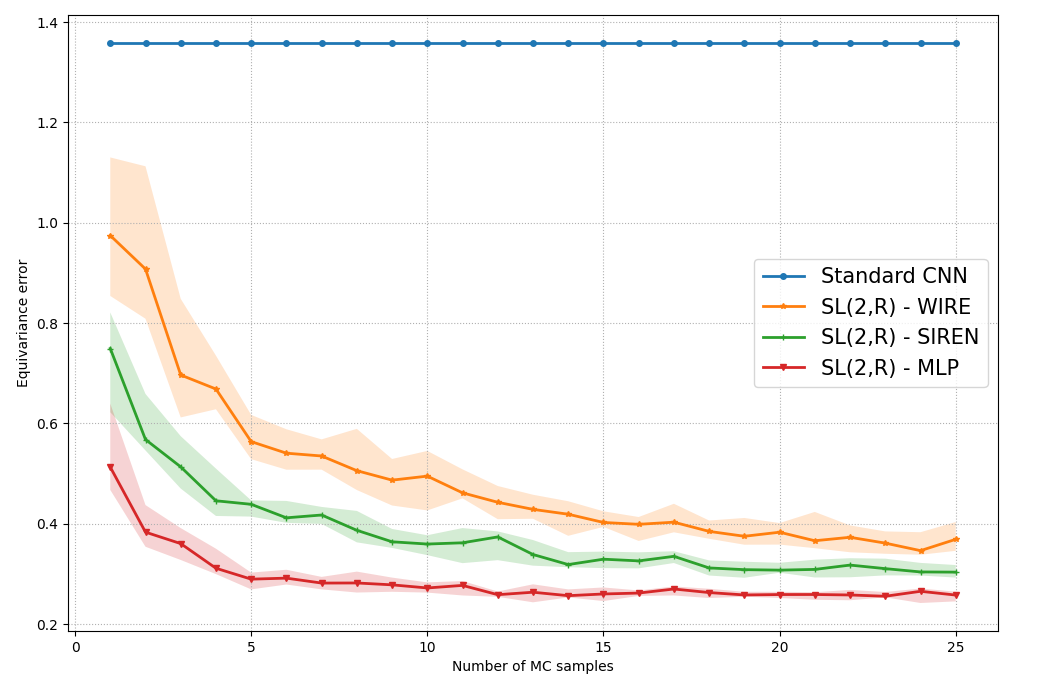}
  \captionsetup{width=.8\linewidth, font=small}
  \caption{Equivariance error as a function of the number of MC samples.}
  \label{fig:equivariance_error_sl2}
\end{figure}

We compare three possible choices of kernel parametrizations, namely a standard MLP with Swish non-linearities as employed by~\cite{finzi2020generalizing}, as well as the SIREN~\cite{sitzmann2020implicit} and WIRE~\cite{saragadam2023wire} INRs.
Note that this choice will also have an effect on the equivariance error, as we are working with a discrete pixel grid when representing images, and any symmetry breaking operations will propagate the loss of equivaraince through the network.
Figure~\ref{fig:test_error_mc_affnist} quantifies the degree to which the performance of the model described in the previous sections is affected by the number of MC samples used when approximating the convolution/cross-correlation integral.
In general, we observe significant performance degradation when employing $\le 6$ samples and as in previous work on integral approximations of continuous convolutions~\citet{knigge2022exploiting} observe no additional benefits beyond $12-14$ samples. 
However, an exact specification of the approximation bounds corresponding to the groups employed is missing in our presentation. 
\begin{figure}[h]
\centering
  \includegraphics[width=0.6\linewidth]{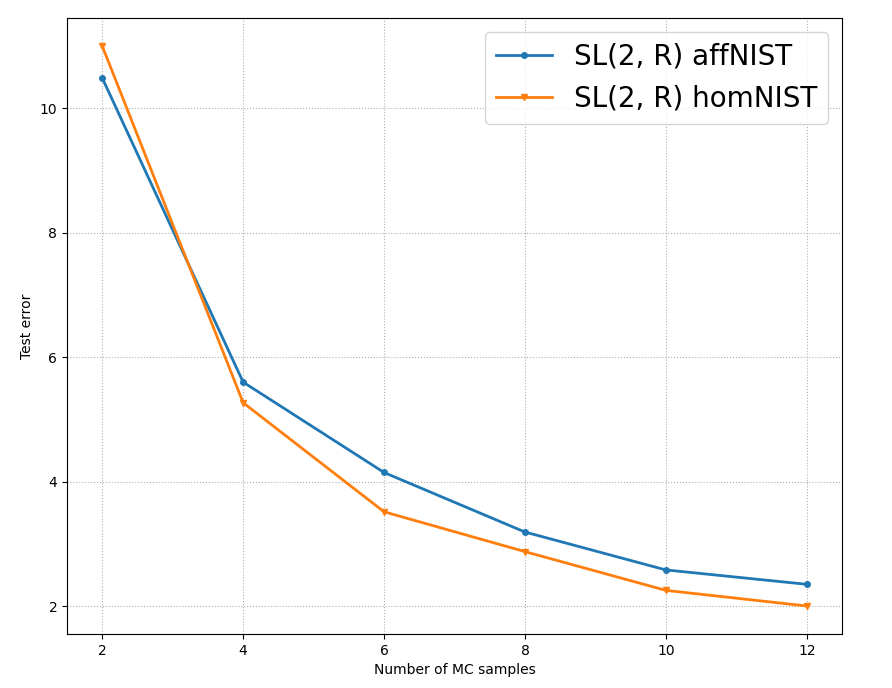}
  \captionsetup{width=.8\linewidth, font=small}
  \caption{Test error on affNIST/homNIST as a function of MC samples.}
  \label{fig:test_error_mc_affnist}
\end{figure}
}
\label{sec:experiment_details}
\end{document}